\documentclass[runningheads]{llncs}
\usepackage[T1]{fontenc}
\usepackage{graphicx}
\usepackage{amsmath,amssymb}
\usepackage{mathrsfs}
\usepackage{booktabs}

\usepackage[linesnumbered,ruled]{algorithm2e}

\usepackage{multirow}
\usepackage{subfigure}
\usepackage[graphicx]{realboxes}
\usepackage{diagbox}
\usepackage{booktabs}

\usepackage{pgfplots}
\usetikzlibrary{patterns,arrows,positioning,calc,fadings,shapes,decorations.markings}
\usepackage{tikz}
\usepackage{apxproof}

\newcolumntype{P}[1]{>{\centering\arraybackslash}p{#1}}

\begin{document}




\newcommand{\PV}{{\rm P{\scshape os}V{\scshape er}}}
\newcommand{\NV}{{\rm N{\scshape ec}V{\scshape er}}}

\newcommand{\IPVad}{{\rm IP{\scshape os}V{\scshape er}}$^\ad$}
\newcommand{\IPVpr}{{\rm IP{\scshape os}V{\scshape er}}$^\pr$}
\newcommand{\IPVco}{{\rm IP{\scshape os}V{\scshape er}}$^\co$}
\newcommand{\INVpr}{{\rm IN{\scshape ec}V{\scshape er}}$^\pr$}
\newcommand{\INVgr}{{\rm IN{\scshape ec}V{\scshape er}}$^\gr$}

\newcommand{\A}{\mathcal{A}}
\newcommand{\R}{\mathcal{R}}
\newcommand{\U}{\mathcal{U}}
\newcommand{\C}{\mathcal{C}}
\newcommand{\fc}{$C$}

\newcommand{\cf}{{\tt cf}}
\newcommand{\ad}{{\tt ad}}
\newcommand{\sta}{{\tt st}}
\newcommand{\co}{{\tt co}}
\newcommand{\gr}{{\tt gr}}
\newcommand{\pr}{{\tt pr}}

\newcommand{\fiba}{\overline{\Gamma}}

\newcommand{\norma}{shared}
\newcommand{\normb}{common}

\newcommand{\stability}{{\rm {\scshape stability}}}
\newcommand{\wstability}{{\rm {\scshape weak-stability}}}

\newcommand{\relevance}{{\rm {\scshape relevance}}}
\newcommand{\srelevance}{{\rm {\scshape strong-relevance}}}

\newcommand{\true}{true}
\newcommand{\false}{false}

\newcommand{\unatt}{\sim}

\newtheorem{fact}{Fact}

\title{Relevance for Stability of Verification Status of a Set of Arguments in Incomplete Argumentation Frameworks (with Proofs)}
\titlerunning{Relevance for Stability of Verification Status in IAFs}
%
\author{Anshu Xiong\inst{1,2} \and
Songmao Zhang\inst{1}}
\authorrunning{A. Xiong et al.}
%
\institute{State Key Laboratory of Mathematical Sciences, Academy of Mathematics and Systems Science, Chinese Academy of Sciences, Beijing 100190, China\and University of Chinese Academy of Sciences, Beijing 100049, China \\ 
\email{xionganshu21@mails.ucas.ac.cn, smzhang@math.ac.cn}}
\maketitle              
\begin{abstract}
The notion of relevance was proposed for stability of justification status of a single argument in incomplete argumentation frameworks (IAFs) in 2024 by Odekerken et al. To extend the notion, we study the relevance for stability of  verification status of a set of arguments in this paper, i.e., the uncertainties in an IAF that have to be resolved in some situations so that answering whether a given set of arguments is an extension obtains the same result in every completion of the IAF. Further we propose the notion of strong relevance for describing the necessity of resolution in all situations reaching stability. An analysis of complexity reveals that detecting the (strong) relevance for stability of sets of arguments can be accomplished in P time under the most semantics discussed in the paper. We also discuss the difficulty in finding tractable methods for relevance detection under grounded semantics.
\keywords{Abstract argumentation \and Incomplete knowledge \and Relevance \and Computational complexity}
\end{abstract}
\section{Introduction}
    As a groundbreaking work for modeling argumentation, Dung’s {\it Abstract Argumentation Framework} (AF) \cite{dung1995acceptability}  represents attack relations among a group of arguments where justified sets of arguments can be computed under various types of semantics. Argumentation can be viewed as a dynamic process where not all arguments and attacks are known in advance. {\it Incomplete Argumentation Framework (IAF)} \cite{baumeister2018verification} is a prominent proposal for introducing qualitative uncertainty to AF for modeling dynamic argumentation settings. An IAF contains not only certain arguments and attacks but also uncertain ones whose existence is unknown at the moment. By investigating the existence of each uncertain element, various AFs can be reached called {\it completions} of the IAF.

    Initially studied in \cite{testerink2019method,odekerken2022approximating}, {\it stability} refers to that the question of interest has the same answer in any situation that will happen in the future, for instance any completion in the IAF setting. Given some semantics $\sigma$, issues about stability in IAFs have been concerned with different goals including stability of: justification of an argument $a$ \cite{mailly2020stability,baumeister2021acceptance,odekerken2023justification}, which refers to acceptable status of $a$ ranging over $\{sceptical,credulous\}\times\{in, out, undec\}$ under $\sigma$; verification of a set of arguments $S$ \cite{baumeister2018verification,fazzinga2020revisiting}, which tells whether $S$ is a $\sigma$-extension; and the whole extensions \cite{xiong2024stability}, which refers to the set of all extensions under $\sigma$ of an AF.
    
    Practically, an IAF being stable means that there is no need to investigate the existence of the current uncertain elements. Meanwhile, the notion of stability also indicates that during the dynamics process from an IAF towards one of its completions, as long as we desire a fixed answer for the question of interest from the completion, then we need to resolve the current uncertainties until stability is reached. Therefore, it is of importance to explore how to resolve uncertainties in IAF so as to reach stability. Such a problem is recently studied in \cite{odekerken2023justification} where the notion of {\it relevance} is proposed. Relevance characterizes which uncertainties need to be resolved in order to reach stability. For instance, addition of an uncertain argument (or attack) $e$ is said to be relevant w.r.t. IAF $I$ if there exists an IAF 'specified' from $I$ where the existence of all uncertain elements is decided except $e$, while the stability of interest is not reached yet and adding $e$ will lead the reaching.
    
    In \cite{odekerken2023justification}, the relevance problem is discussed for stability of justification status of a given argument. We believe that relevance is worth exploring for other types of stability, and focus on stability of verification status of a given set of arguments in this paper. Essentially, this relevance is distinctive from the relevance in \cite{odekerken2023justification} due to the inherent differences across the two types of stability. Further, for an uncertain element whose addition or removal is relevant, it is also interesting to validate whether such an action is necessary to be done in any situation for reaching stability. We define such a stricter form as {\it strong relevance}. Generally, relevance reveals the need of resolutions in some situation for reaching stability, whereas strong relevance describes the necessity of resolutions in all situations.

    The contribution of this paper consists of (1) extending the relevance problem proposed in \cite{odekerken2023justification} to relevance for stability of verification status of a set of arguments; and (2) proposing the notion of strong relevance for stability in IAFs. Specifically, we analyze the complexity for deciding two relevance problems under five common semantics including admissible, stable, complete, grounded and preferred semantics. For admissible, stable and complete semantics, we find tractable methods for detecting both relevance and strong relevance, showing that they are all in $P$ complexity class. For preferred semantics, we give the upper bound and prove the lower bound of two relevance problems by constructing reductions from known problems in IAFs and SAT problems. The difficulty of finding tractable method for solving relevance problems under grounded semantics is also discussed. Overall our results show that deciding relevance for stability of verification of a set of arguments is simpler than that of justification of an argument, whose complexity is up to exponential level under all the common semantics as proven in \cite{odekerken2023justification}. 

    The paper is organized as follows. In Section 2, we provide a brief, necessary introduction to basic definitions about AF, IAF and stability problems of verification in IAFs. In Section 3, we define and give a complexity analysis of the relevance problem for stability of verification status of a set of arguments. In Section 4, we introduce strong relevance and also study the complexity of its decision problems. Then in Section 5, the related work as well as related issues are discussed including the difficulty in tackling the relevance problem under grounded semantics. Lastly the paper is concluded in Section 6.
\section{Preliminaries}
\paragraph{Argumentation Framework (AF).}An {\it abstract argumentation framework} \cite{dung1995acceptability} is a directed graph $F=\langle \A, \R\rangle$ where $\A$ represents a set of considered arguments and $\R\subseteq\A\times\A$ the set of attacks between arguments in $\A$. We say that $a$ {\it attacks} $b$ if $(a,b)\in\R$, $a$ {\it attacks} $S\subseteq\A$ if $\exists b\in S,(a,b)\in\R$, and the meaning of $S$ {\it attacks} $a$ is analogous. Given a set of arguments $S\subseteq\A$, let $S^+_F=\{a\in\A\mid S$ attacks $a\}$ and $S^-_F=\{a\in\A\mid a$ attacks $S\}$, and $S$ is {\it conflict-free} iff $S\cap S^+_F=\emptyset$. We say that $S$ {\it defends} $a$ if all the attackers of $a$ are attacked by $S$ and use the so-called {\it characteristic function} $\Gamma_F(S)$ \cite{dung1995acceptability} to denote all the arguments that $S$ defends, i.e., $\Gamma_F(S)=\{a\in\A\mid S$ defends $a$ in $F\}$.
    \paragraph{Semantics of AF.}A {\it semantics} $\sigma$ is a function of which the input is an AF $F=\langle \A, \R\rangle$ and the output $\sigma(F)$ a set of subsets of $\A$, where every element of $\sigma(F)$ is called a {\it $\sigma$-extension} of $F$. The common semantics {admissible, stable, complete, grounded and preferred} semantics (abbr. $\ad$, $\sta$, $\co$, $\gr$, $\pr$) are originally proposed in \cite{dung1995acceptability} and defined as follows.
\begin{definition}\label{sem}
    Given an AF $F=\langle \A, \R\rangle$ and $S\subseteq\A$,\\
    1. $S\in\ad(F)$ iff $S$ is conflict-free and $S\subseteq\Gamma_F(S)$;\\
    2. $S\in\sta(F)$ iff $S$ is conflict-free and $S^+_F=\A\setminus S$;\\
    3. $S\in\co(F)$ iff $S\in\ad(F)$ and $\forall a\in\Gamma_F(S), a\in S$;\\
    4. $S\in\gr(F)$ iff $S$ is $\subseteq$-minimal in $\co(F)$; and\\
    5. $S\in\pr(F)$ iff $S$ is $\subseteq$-maximal in $\ad(F)$.\\
\end{definition}
    Given an AF $F$, a set of arguments $S$ of $F$ and semantics $\sigma$\footnote{In this paper we limit the semantics discussed to the  
    five common semantics in Definition \ref{sem}, and unless specifically stated otherwise the given semantics $\sigma$ in context ranges over $\{\ad$,$\sta$,$\co$,$\gr$,$\pr\}$.}, the notion of {\it verification status} is defined as follows: (the verification status of) $S$ is $\sigma$-$\true$ (resp., -$\false$) iff $S\in\sigma(F)$ (resp., $S\notin\sigma(F)$).
    
    Now we recall the notion of IAF expanding AF with qualitative uncertainty. 
\begin{definition}[Incomplete Argumentation Framework \cite{baumeister2018verification}]
      An incomplete argumentation framework (IAF) is a quadruple $\langle\A, \A^?, \R, \R^?\rangle$ where $\A$ and $\A^?$ are disjoint sets of arguments and $\R$ and $\R^?$ disjoint subsets of $(\A\cup\A^?)\times(\A\cup\A^?)$. $\A$ (resp., $\R$) represents arguments (resp., attacks) that are known to certainly exist, while $\A^?$ (resp., $\R^?$) contains additional arguments (resp., attacks) whose existence is yet uncertain. An IAF is called an AtIAF (resp., ArIAF) if it has no uncertain arguments (resp., attacks). 
\end{definition}\label{part}
    The {\it partial completions} \cite{odekerken2023justification} of an IAF $I$ represent the possible IAFs that $I$ can be specified to be, i.e., by remaining the certain parts and deciding some of the uncertain elements to be existent or not.
\begin{definition}[Partial completion]\cite{odekerken2023justification}
    Given an IAF $I=\langle\A,\A^?,\R,\R^?\rangle$, a partial completion is an IAF $I'=\langle\A',\A^{?'},\R',\R^{?'}\rangle$ where:\\
    - $\A\subseteq \A' \subseteq \A\cup\A^?$;\\
    - $\R\cap (\A'\cup\A^{?'})\times(\A'\cup\A^{?'})\subseteq \R' \subseteq \R\cup\R^?$;\\
    - $\A^{?'}\subseteq \A^?$; and\\
    - $\R^{?'}\subseteq \R^?$.
  \end{definition}  
    Note that $\A'\cap\A^{?'}=\emptyset, \R'\cap\R^{?'}=\emptyset, \R'\subseteq(\A'\cup\A^{?'})\times(\A'\cup\A^{?'})$ and $\R^{?'}\subseteq(\A'\cup\A^{?'})\times(\A'\cup\A^{?'})$ because $I'$ is an IAF. We use $part(I)$ to denote the set containing all of the partial completions of $I$, and $cert(I)=\langle\A,\R\cap(\A\times\A)\rangle$ to denote the AF projected on the certain parts of $I$. 
    We can see that by partial completion, the notion of $\it completion$ \cite{baumeister2018verification} can be alternatively defined as that an AF $F$ is a completion of an IAF $I$ iff there is a partial completion $I'$ of $I$ such that $F=cert(I')$.

    Given an IAF $I=\langle\A,\A^?,\R,\R^?\rangle$ and a set of arguments $S\subseteq\A\cup\A^?$, similarly to notations for AF, we give the following notations:\\
    - $S^+_I =\{ a\in \A\cup\A^?\mid \exists b \in S,(b,a)\in \R \}$;\\
    - $S^-_I=\{ a\in \A\cup\A^?\mid \exists b \in S,(a,b)\in \R \}$; and\\
    - $S^\unatt_I=\{ a\in \A\cup\A^?\mid \forall b \in S,(b,a)\notin\R\cup\R^?\}$.
    
    Further, in order to concisely describe the changing of IAF $I$, given a set of attacks $\R_0\subseteq\R^?$ or a set of arguments $\A_0\subseteq\A^?$, let:\\
    - $I+\R_0=\langle\A,\A^?,\R\cup\R_0,\R^?\setminus\R_0\rangle$;\\
    - $I-\R_0=\langle\A,\A^?,\R,\R^?\setminus\R_0\rangle$;\\
    - $I+\A_0=\langle\A\cup\A_0,\A^?\setminus\A_0,\R,\R^?\rangle$; and\\
    - $I-\A_0=\langle\A,\A^?\setminus\A_0,\R\setminus\R',\R^?\setminus\R'\rangle$, where $\R'=\{(a,b)\in\R\cup\R^?\mid a\in\A_0$ or $b\in\A_0\}$, i.e., the set of attacks related to the arguments in $\A_0$.

    Now we introduce the notion of stability of verification of a set of arguments, which is defined based on verification status. Given an IAF $I$, a set of arguments $S$ and a verification status $j\in \{\ad,\sta,\co,\gr,\pr\}\times\{\true,\false\}$, $S$ is stable-$j$ w.r.t. $I$ iff the verification status of $S$ remains $j$ in any completion of $I$. 
\begin{definition}[Stability of verification in IAFs]
    Given an IAF $I=\langle\A,\A^?,\R,$ $\R^?\rangle$, a set of arguments $S\subseteq\A\cup\A^?$, and a verification status $j\in \{\ad,\sta,\co,\gr,\pr\}$ $\times\{\true,\false\}$, $S$ is stable-$j$ w.r.t. $I$ iff for any completion $F$ of $I$, $S$ is $j$ in $F$. 
\end{definition}
    The precise complexity results of verification stability problems can be obtained directly from the results of $PosVer$ and $NecVer$ problems given in \cite{baumeister2018verification,fazzinga2020revisiting}. For an IAF $I$, a set of arguments $S$ and semantics $\sigma$, $PosVer_\sigma$ asks whether there is a completion $F$ of $I$ such that $S\in\sigma(F)$ whereas $NecVer_\sigma$ asks whether for any completion $F$ of $I$, $S\in\sigma(F)$ always holds. Therefore, $S$ is stable-$\sigma$-$\true$ w.r.t. $I$ iff $NecVer_\sigma(I,S)=true$ and $S$ is stable-$\sigma$-$\false$ w.r.t. $I$ iff $PosVer_\sigma(I,S)=false$ holds. For $\sigma\in\{\ad,\sta,\co,\gr\}$, both $PosVer_\sigma$ and $NecVer_\sigma$ problems are in $P$ complexity class, while $PosVer_\pr$ is $\Sigma_2$-$c$ and $NecVer_\pr$ is $coNP$-$c$ \cite{baumeister2018verification,fazzinga2020revisiting}. These known results will be used subsequently in this paper. 
    
    Given semantics $\sigma$, we say that a set of arguments $S$ is stable-$\sigma$ w.r.t. an IAF $I$ if $S$ is stable-$\sigma$-$\true$ or stable-$\sigma$-$\false$ w.r.t. $I$, which means that $S$ holds the same verification result under $\sigma$ semantics in any completion of the IAF $I$.
\section{Relevance for Stability of Verification}\label{def1}
    In this section, we define relevance for stability of verification status of a given set of arguments and then study the complexity of related computational issues. 
\subsection{Definitions for relevance}
    First, we extend the notion of relevance to stability of verification. Given an IAF $I$, a set of arguments $S$, addition (resp., removal) of an uncertain element $e$ is relevant for $S$ to reach $j$ verification status if there is a partial completion of $I$ where $e$ is the unique uncertain element not decided yet, and adding (resp., removing) $e$ will lead $S$ to become $j$ whereas the opposite action will not.
\begin{definition}[Relevance for stability of verification]\label{rel}
    Given an IAF $I=\langle\A,\A^?,\R,\R^?\rangle$, a set of arguments $S\subseteq\A\cup\A^?$, a verification status $j$ and an uncertain element $e\in\A^?\cup\R^?$, addition (resp., removal) of $e$ is $j$-relevant for $S$ w.r.t. $I$ iff there exists an IAF $I'=\langle\A',\A^{?'},\R',\R^{?'}$ $\rangle\in part(I)$ satisfying that $\A^{?'}\cup\R^{?'}=\{e\}$, and $S$ is $j$ in $cert(I'+\{e\})$ (resp., $cert(I'-\{e\})$) while $S$ is not $j$ in $cert(I'-\{e\})$ (resp., $cert(I'+\{e\})$). We use $RE^+(I,S,j)$ and $RE^-(I,S,j)$ to respectively denote the uncertain elements whose addition and removal are $j$-relevant for $S$ w.r.t. $I$.
\end{definition}
    In other words, addition or removal of $e$ is $j$-relevant for $S$ if such an action is needed for $S$ to reach $j$ status in some situation. Since the verification status of $S$ is either -$\true$ or -$\false$ in completions, we can see that the addition and removal are dual actions according to Definition \ref{rel}, i.e., $RE^+(I,S,\sigma$-$\true)=RE^-(I,S,\sigma$-$\false)$ and $RE^-(I,S,\sigma$-$\true)=RE^+(I,S$ $,\sigma$-$\false)$ always hold. This says that for any uncertain element, the -$\true$ relevance of its addition or removal coincides with the -$\false$ relevance of the opposite action. 

    If neither addition nor removal of $e$ is $\sigma$-$\true$($\false$)-relevant for $S$ w.r.t. $I$, we say that $e$ is $\sigma$-{\it irrelevant}. This means that there is no need to investigate $e$ to reach stability, as in all situations, the existence of $e$ has no influence on the verification result of $S$. 

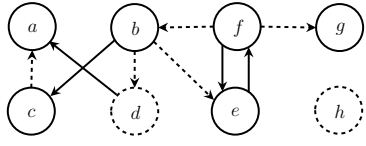
\begin{figure}[t]
\centering 
\tikzset{every picture/.style={line width=0.75pt}} 
\begin{tikzpicture}[scale=0.0125]
\draw   (129,199) .. controls (129,185.19) and (140.19,174) .. (154,174) .. controls (167.81,174) and (179,185.19) .. (179,199) .. controls (179,212.81) and (167.81,224) .. (154,224) .. controls (140.19,224) and (129,212.81) .. (129,199) -- cycle ;
\draw   (239,288) .. controls (239,274.19) and (250.19,263) .. (264,263) .. controls (277.81,263) and (289,274.19) .. (289,288) .. controls (289,301.81) and (277.81,313) .. (264,313) .. controls (250.19,313) and (239,301.81) .. (239,288) -- cycle ;
\draw   (348,289) .. controls (348,275.19) and (359.19,264) .. (373,264) .. controls (386.81,264) and (398,275.19) .. (398,289) .. controls (398,302.81) and (386.81,314) .. (373,314) .. controls (359.19,314) and (348,302.81) .. (348,289) -- cycle ;
\draw   (457,288) .. controls (457,274.19) and (468.19,263) .. (482,263) .. controls (495.81,263) and (507,274.19) .. (507,288) .. controls (507,301.81) and (495.81,313) .. (482,313) .. controls (468.19,313) and (457,301.81) .. (457,288) -- cycle ;
\draw   (346,199) .. controls (346,185.19) and (357.19,174) .. (371,174) .. controls (384.81,174) and (396,185.19) .. (396,199) .. controls (396,212.81) and (384.81,224) .. (371,224) .. controls (357.19,224) and (346,212.81) .. (346,199) -- cycle ;
\draw   (130,289) .. controls (130,275.19) and (141.19,264) .. (155,264) .. controls (168.81,264) and (180,275.19) .. (180,289) .. controls (180,302.81) and (168.81,314) .. (155,314) .. controls (141.19,314) and (130,302.81) .. (130,289) -- cycle ;
\draw  [dash pattern={on 1.5pt off 1.5pt}] (239,198.56) .. controls (239.24,184.76) and (250.63,173.76) .. (264.44,174) .. controls (278.24,174.24) and (289.24,185.63) .. (289,199.44) .. controls (288.76,213.24) and (277.37,224.24) .. (263.56,224) .. controls (249.76,223.76) and (238.76,212.37) .. (239,198.56) -- cycle ;
\draw  [dash pattern={on 1.5pt off 1.5pt}] (456,200) .. controls (456,186.19) and (467.19,175) .. (481,175) .. controls (494.81,175) and (506,186.19) .. (506,200) .. controls (506,213.81) and (494.81,225) .. (481,225) .. controls (467.19,225) and (456,213.81) .. (456,200) -- cycle ;
\draw    (244.83,216.5) -- (175.19,271.64) ;
\draw [shift={(172.83,273.5)}, rotate = 321.63] [fill={rgb, 255:red, 0; green, 0; blue, 0 }  ][line width=0.08]  [draw opacity=0] (10.72,-5.15) -- (0,0) -- (10.72,5.15) -- (7.12,0) -- cycle    ;
\draw    (243.33,275) -- (176.6,216.97) ;
\draw [shift={(174.33,215)}, rotate = 41.01] [fill={rgb, 255:red, 0; green, 0; blue, 0 }  ][line width=0.08]  [draw opacity=0] (10.72,-5.15) -- (0,0) -- (10.72,5.15) -- (7.12,0) -- cycle    ;
\draw    (385.33,219) -- (385.33,264) ;
\draw [shift={(385.33,267)}, rotate = 270] [fill={rgb, 255:red, 0; green, 0; blue, 0 }  ][line width=0.08]  [draw opacity=0] (10.72,-5.15) -- (0,0) -- (10.72,5.15) -- (7.12,0) -- cycle    ;
\draw    (357.33,269) -- (357.33,223) ;
\draw [shift={(357.33,220)}, rotate = 90] [fill={rgb, 255:red, 0; green, 0; blue, 0 }  ][line width=0.08]  [draw opacity=0] (10.72,-5.15) -- (0,0) -- (10.72,5.15) -- (7.12,0) -- cycle    ;
\draw  [dash pattern={on 1.5pt off 1.5pt}]  (284.33,273) -- (346.18,213.09) ;
\draw [shift={(348.33,211)}, rotate = 135.91] [fill={rgb, 255:red, 0; green, 0; blue, 0 }  ][line width=0.08]  [draw opacity=0] (10.72,-5.15) -- (0,0) -- (10.72,5.15) -- (7.12,0) -- cycle    ;
\draw  [dash pattern={on 1.5pt off 1.5pt}]  (264,263) -- (264,227) ;
\draw [shift={(264,224)}, rotate = 90] [fill={rgb, 255:red, 0; green, 0; blue, 0 }  ][line width=0.08]  [draw opacity=0] (10.72,-5.15) -- (0,0) -- (10.72,5.15) -- (7.12,0) -- cycle    ;
\draw  [dash pattern={on 1.5pt off 1.5pt}]  (398,289) -- (454,288.05) ;
\draw [shift={(457,288)}, rotate = 179.03] [fill={rgb, 255:red, 0; green, 0; blue, 0 }  ][line width=0.08]  [draw opacity=0] (10.72,-5.15) -- (0,0) -- (10.72,5.15) -- (7.12,0) -- cycle    ;
\draw  [dash pattern={on 1.5pt off 1.5pt}]  (348,289) -- (292,288.05) ;
\draw [shift={(289,288)}, rotate = 0.97] [fill={rgb, 255:red, 0; green, 0; blue, 0 }  ][line width=0.08]  [draw opacity=0] (10.72,-5.15) -- (0,0) -- (10.72,5.15) -- (7.12,0) -- cycle    ;
\draw  [dash pattern={on 1.5pt off 1.5pt}]  (154,224) -- (154.93,261) ;
\draw [shift={(155,264)}, rotate = 268.57] [fill={rgb, 255:red, 0; green, 0; blue, 0 }  ][line width=0.08]  [draw opacity=0] (10.72,-5.15) -- (0,0) -- (10.72,5.15) -- (7.12,0) -- cycle    ;

\draw (154.22,196.69) node  [font=\Large,xscale=0.5,yscale=0.5]  {$c$};
\draw (264.66,286.69) node  [font=\Large,xscale=0.5,yscale=0.5]  {$b$};
\draw (374.73,287.69) node  [font=\Large,xscale=0.5,yscale=0.5]  {$f$};
\draw (483.11,286.69) node  [font=\Large,xscale=0.5,yscale=0.5]  {$g$};
\draw (371.22,197.69) node  [font=\Large,xscale=0.5,yscale=0.5]  {$e$};
\draw (155.43,287.69) node  [font=\Large,xscale=0.5,yscale=0.5]  {$a$};
\draw (265.11,197.69) node  [font=\Large,xscale=0.5,yscale=0.5]  {$d$};
\draw (482.11,198.69) node  [font=\Large,xscale=0.5,yscale=0.5]  {$h$};
\end{tikzpicture}

\caption{An example IAF $I_{ex}$ (uncertain arguments and attacks are depicted using dashed circles and lines, respectively)}
\label{fig1}
\vspace{-5mm}
\end {figure}
\begin{example}
    Consider the IAF $I_{ex}$ in Figure \ref{fig1}. For the set of arguments $S=\{a,b\}$, $RE^+(I,S,\ad$-$\true)=\{(b,d)\}$ and $RE^-(I,S,\ad$-$\true)=\{d,(f,b)\}$. The uncertain attack $(c,a)$ is $\ad$-irrelevant, since $c$ is necessarily attacked by $b\in S$ and $S$ will be defended against the attacks from $c$ in any completion. In contrary to $(c,a)$, removal of $d$ is $\ad$-$\true$-relevant for $S$ w.r.t. $I_{ex}$, since when $(b,d)$ is removed, there are no arguments in $S$ attacking $d$ and $d$ is needed to be removed for $S$ to be admissible. One can see that there is a partial completion $I_{ex}'=I_{ex}-\{(c,a),(b,d),(b,e),(f,b),(f,g)\}-\{h\}$ of $I_{ex}$ such that $S$ is not stable-$\ad$ w.r.t. $I_{ex}'$ whereas $S$ is stable-$\ad$-$\true$ w.r.t. $I_{ex}'-\{d\}$.
\end{example}
    Lastly in this subsection, we give a proposition that relates the identification of stability to detecting the existence of relevant elements. 
\begin{proposition}\label{rel_sta_relation}
    Given an IAF $I=\langle\A,\A^?,\R,\R^?\rangle$, a set of arguments $S\subseteq\A\cup\A^?$ and semantics $\sigma$, $S$ is stable-$\sigma$ w.r.t. $I$ iff $\forall e\in\A^?\cup\R^?$, $e$ is $\sigma$-irrelevant.
\end{proposition}
\begin{proof}
    ($\Rightarrow$) Assume that there is an element $e\in\A^?\cup\R^?$ such that $e$ is not $\sigma$-irrelevant, then there is an IAF $I'\in part(I)$ such that $S$ is not stable-$\sigma$ w.r.t. $I'$ according to the definition of relevance, hence $S$ is not stable-$\sigma$ w.r.t. $I'$ neither, which yields contradiction.
    
    ($\Leftarrow$) Assume that $S$ is not stable-$\sigma$, then there is an IAF $I'\in part(I)$ and an uncertain element $e$ of $I'$ such that $S$ is not stable-$\sigma$ w.r.t. $I'$ whereas deciding $e$ make $I'$ become stable. W.L.O.G., assume that $S$ is stable-$\sigma$-$\true$ w.r.t. $I'+\{e\}$, let $C$ be a completion of $I'-\{e\}$ that $S$ is $\sigma$-$\false$ in $C$, and let $I''$ be the IAF s.t. $I''\in part(I')$, $e$ is the unique uncertain element of $I''$, and $C=cert(I''-\{e\})$. We can see that $I''+\{e\}\in part(I'+\{e\})$, hence $S$ is $\sigma$-$\true$ in $cert(I''+\{e\})$. Therefore, addition of $e$ is $\sigma$-$\true$-relevant for $S$ w.r.t. $I$, a contradiction occurs.
\end{proof}
\subsection{Complexity analysis of relevance}
    In this subsection we study the computational complexity of decision problems for identifying the relevance of adding or removing an uncertain element for stability of verification status of a given set of arguments, which is formulated as follows.
\begin{flushleft}
\begin{tabular}{p{1\textwidth}}
\toprule
$j$-\relevance\ of action $\textbf{a}$ \\
\hline
\textbf{Given:}\ \ An IAF $I=\langle\A,\A^?,\R,\R^?\rangle$, a set of arguments $S\subseteq\A\cup\A^?$, a verification status $j$, an action $\textbf{a}\in\{$addition,removal$\}$ and an uncertain element $\textbf{e}\in\A^?\cup\R^?$. \\
\textbf{Question:}\ \ Is \textbf{a} of $\textbf{e}$ $j$-relevant for $S$ w.r.t. $I$? \\
\bottomrule
\end{tabular}
\end{flushleft}

    Next we give results for obtaining upper bounds of \relevance\ problem. Directly according to Definition \ref{rel}, we can conclude that for semantics $\sigma$ whose verification problem is in $C$ complexity, the relevance problem under $\sigma$ is in $NP^C$. Thus we directly obtain an upper bound for various semantics as follows.
\begin{corollary}
    The following results hold:\\
    1. for $j\in\{\ad,\sta,\co,\gr\}\times\{\true,\false\}$, $j$-\relevance\ is in $NP$; and\\
    2. $\pr$-$\true$($\false$)-\relevance\ is in $\Sigma_2^p$.
\end{corollary}
    Although the above corollary tells an exponential upper bound for all semantics, in the next we will give characterizing conditions for $\ad$,$\sta$ and $\co$ semantics, all of which can be checked in polynomial time. 
    
    To obtain the complexity of \relevance\ problem, we consider one of its specific problems which constrains the input IAF to have none uncertain arguments, i.e., the input being an AtIAF. The following proposition shows that the general \relevance\ problem can be reduced to this specific form, which means that both of them hold the same complexity. The way in this proposition for constructing a related AtIAF from an IAF is borrowed from \cite{mantadelis2020probabilistic}.  
\begin{proposition}\label{att-reduce}
    Given an IAF $I=\langle\A,\A^?,\R,\R^?\rangle$, a set of arguments $S\subseteq\A\cup\A^?$ and a verification status $j$, let $w$ be an argument such that $w\notin\A\cup\A^?$ and $I_{att}=\langle\A\cup\A^?\cup\{w\},\emptyset,\R,\R^?\cup\{(w,a)\mid a\in\A^?\}\rangle$. The following results hold: \\ 
    - for each uncertain attack $e\in\R^?$, $e\in RE^+(I,S,j)$ (resp., $e\in RE^-(I,S,j)$) iff $e\in RE^+(I_{att},S\cup\{w\},j)$ (resp., $e\in RE^-(I_{att},S\cup\{w\},j)$); and\\
    - for each uncertain argument $e\in\A^?$, $e\in RE^+(I,S,j)$ (resp., $e\in RE^-(I,S,j)$) iff $(w,e)\in RE^-(I_{att},S\cup\{w\},j)$ (resp., $(w,e)\in RE^+(I_{att},S\cup\{w\},j)$).
\end{proposition}
\begin{proof}
    We prove the case $e\in RE^+(I,S,j)$ where $j=\sigma\times\true$ and $\sigma\in\{\ad,\sta,\co,\gr,\pr\}$ for both conclusions in the proposition and the proof of case $e\in RE^-(I,S,j)$ or $j=\sigma\times\false$ is analogous.
    
    1. ($\Rightarrow$) Since $e\in RE^+(I,S,\sigma$-$\true)$, there is a completion $F=\langle\A_F,\R_F\rangle$ of $I$ such that $S\in\sigma(F),e\in\R_F$ and removal of $e$ will lead $S$ fail to be a $\sigma$-extension. Let $F'=\langle\A_F\cup\A^?,\R_F\cup\R\cup\{(w,v)\mid v\in\A^?\setminus\A_F\}\rangle$. Then it is easy to see that $S\cup\{w\}\in\sigma(F')$ by the property of $\sigma$ semantics and $F'$ is also a completion of $I_{att}$. Meanwhile, removal of $e$ will lead $S$ fail to be a $\sigma$-extension in $F'_{\downarrow\A_F}$\footnote{Given an AF $F=\langle\A,\R\rangle$ and a set of arguments $S\subseteq\A$, we use $F_{\downarrow S}$ to denote $\langle S,\R\cap(S\times S)\rangle$.} then cause that $S\cup\{w\}$ fails to be a $\sigma$-extension in $F'$, hence $e\in RE^+(I_{att},S\cup\{w\},\sigma$-$\true)$ holds.

    ($\Leftarrow$) Since $e\in RE^+(I_{att},S\cup\{w\},\sigma$-$\true)$, there is a completion $F=\langle\A_F,\R_F\rangle$ of $I_{att}$ such that $S\in\sigma(F),e\in\R_F$ and removal of $e$ will lead $S$ fail to be a $\sigma$-extension. And since $e\in\R^?$, $e$ is not related to $w$. Let $F'=F_{\downarrow\A\cup\A^?\setminus(\{w\}\cup\{w\}^+_F)}$. Then we can see that $S\in\sigma(F')$ and removal of $e$ will lead $S$ fail to be a $\sigma$-extension in $F'$. Meanwhile, $F'$ is also a completion of $I$, hence $e\in RE^+(I,S,\sigma$-$\true)$ holds.

    2. ($\Rightarrow$) Since $e\in RE^+(I,S,\sigma$-$\true)$, there is a completion $F=\langle\A_F,\R_F\rangle$ of $I$ such that $S\in\sigma(F),e\in\A_F$ and removal of $e$ will lead $S$ fail to be a $\sigma$-extension. Let $F'=\langle\A_F\cup\A^?,\R_F\cup\R\cup\{(w,v)\mid v\in\A^?\setminus\A_F\}\rangle$. Then it is easy to see that $S\cup\{w\}\in\sigma(F')$ by the property of $\sigma$ semantics and $F'$ is also a completion of $I_{att}$. Meanwhile, addition of attack $(w,e)$ will lead $S\cup\{w\}$ fail to be a $\sigma$-extension in $F'$, hence $e\in RE^+(I_{att},S\cup\{w\},\sigma$-$\true)$ holds.

    ($\Leftarrow$) Since $(w,e)\in RE^-(I_{att},S\cup\{w\},\sigma$-$\true)$, there is a completion $F=\langle\A_F,\R_F\rangle$ of $I_{att}$ that $S\in\sigma(F),(w,e)\notin\R_F$ and addition of $e$ will lead $S$ fail to be a $\sigma$-extension. Let $F'=F_{\downarrow\A\cup\A^?\setminus(\{w\}\cup\{w\}^+_F)}$. Then we can see that $S\in\sigma(F')$ and removal of $e$ will lead $S$ fail to be a $\sigma$-extension in $F'$ due to the property of $\sigma$ semantics. Meanwhile, $F'$ is also a completion of $I$, hence $e\in RE^+(I,S,\sigma$-$\true)$ holds.
\end{proof} 
    Essentially, the uncertainty of the uncertain arguments can be replaced by the uncertainty of an attack from $w$ to them, whereas the uncertainty of the uncertain attacks remains. This is because once an argument $e$ is attacked by an argument not attacked by others, $e$ is impossible to be accepted and can be regarded as absent.
\begin{corollary}\label{att-reduce-cor}
    Given a verification status $j$, $j$-\relevance\ holds the same complexity with $j$-\relevance\ problem constraining that the input IAF is an AtIAF. 
\end{corollary}
    By Corollary \ref{att-reduce-cor}, in the following we limit our discussion to \relevance\ problem whose input IAF is an AtIAF. Recall that the $\sigma$-$\true$-\relevance\ of addition (resp., removal) essentially coincides with $\sigma$-$\false$-\relevance\ of removal (resp., addition), hence we just focus on $\sigma$-$\true$ \relevance. Also, for the considered set of arguments which is already stable, the answer of \relevance\ is trivial since all of the uncertain elements are irrelevant according to Proposition \ref{rel_sta_relation}. Thus, we only need to study methods for finding relevance elements for a given, unstable set of arguments. Note that despite that the method by Proposition \ref{att-reduce} gives a reduction to a specific AtIAF with an argument not attacked by others, the results we will give are about general AtIAFs.

    Given an AtIAF $I$ and a set of arguments $S$, since all semantics we discuss are based on conflict-freeness, the uncertain attacks between the arguments inside $S$ are necessary to be removed for reaching $\sigma$-$\true$ stability of $S$, thus removing each of these attacks is relevant whereas adding is not. 
\begin{figure}[t]
\centering 
\tikzset{every picture/.style={line width=0.75pt}} 
\begin{tikzpicture}[scale=0.0125]
\draw   (75,215) .. controls (75,201.19) and (86.19,190) .. (100,190) .. controls (113.81,190) and (125,201.19) .. (125,215) .. controls (125,228.81) and (113.81,240) .. (100,240) .. controls (86.19,240) and (75,228.81) .. (75,215) -- cycle ;
\draw   (185,304) .. controls (185,290.19) and (196.19,279) .. (210,279) .. controls (223.81,279) and (235,290.19) .. (235,304) .. controls (235,317.81) and (223.81,329) .. (210,329) .. controls (196.19,329) and (185,317.81) .. (185,304) -- cycle ;
\draw   (294,305) .. controls (294,291.19) and (305.19,280) .. (319,280) .. controls (332.81,280) and (344,291.19) .. (344,305) .. controls (344,318.81) and (332.81,330) .. (319,330) .. controls (305.19,330) and (294,318.81) .. (294,305) -- cycle ;
\draw   (403,304) .. controls (403,290.19) and (414.19,279) .. (428,279) .. controls (441.81,279) and (453,290.19) .. (453,304) .. controls (453,317.81) and (441.81,329) .. (428,329) .. controls (414.19,329) and (403,317.81) .. (403,304) -- cycle ;
\draw   (292,215) .. controls (292,201.19) and (303.19,190) .. (317,190) .. controls (330.81,190) and (342,201.19) .. (342,215) .. controls (342,228.81) and (330.81,240) .. (317,240) .. controls (303.19,240) and (292,228.81) .. (292,215) -- cycle ;
\draw   (76,305) .. controls (76,291.19) and (87.19,280) .. (101,280) .. controls (114.81,280) and (126,291.19) .. (126,305) .. controls (126,318.81) and (114.81,330) .. (101,330) .. controls (87.19,330) and (76,318.81) .. (76,305) -- cycle ;
\draw   (185,214.56) .. controls (185.24,200.76) and (196.63,189.76) .. (210.44,190) .. controls (224.24,190.24) and (235.24,201.63) .. (235,215.44) .. controls (234.76,229.24) and (223.37,240.24) .. (209.56,240) .. controls (195.76,239.76) and (184.76,228.37) .. (185,214.56) -- cycle ;
\draw   (402,217) .. controls (402,203.19) and (413.19,192) .. (427,192) .. controls (440.81,192) and (452,203.19) .. (452,217) .. controls (452,230.81) and (440.81,242) .. (427,242) .. controls (413.19,242) and (402,230.81) .. (402,217) -- cycle ;
\draw    (190.83,232.5) -- (121.19,287.64) ;
\draw [shift={(118.83,289.5)}, rotate = 321.63] [fill={rgb, 255:red, 0; green, 0; blue, 0 }  ][line width=0.08]  [draw opacity=0] (10.72,-5.15) -- (0,0) -- (10.72,5.15) -- (7.12,0) -- cycle    ;
\draw    (189.33,291) -- (122.6,232.97) ;
\draw [shift={(120.33,231)}, rotate = 41.01] [fill={rgb, 255:red, 0; green, 0; blue, 0 }  ][line width=0.08]  [draw opacity=0] (10.72,-5.15) -- (0,0) -- (10.72,5.15) -- (7.12,0) -- cycle    ;
\draw    (331.33,235) -- (331.33,280) ;
\draw [shift={(331.33,283)}, rotate = 270] [fill={rgb, 255:red, 0; green, 0; blue, 0 }  ][line width=0.08]  [draw opacity=0] (10.72,-5.15) -- (0,0) -- (10.72,5.15) -- (7.12,0) -- cycle    ;
\draw    (303.33,285) -- (303.33,239) ;
\draw [shift={(303.33,236)}, rotate = 90] [fill={rgb, 255:red, 0; green, 0; blue, 0 }  ][line width=0.08]  [draw opacity=0] (10.72,-5.15) -- (0,0) -- (10.72,5.15) -- (7.12,0) -- cycle    ;
\draw  [dash pattern={on 1.5pt off 1.5pt}]  (230.33,289) -- (292.18,229.09) ;
\draw [shift={(294.33,227)}, rotate = 135.91] [fill={rgb, 255:red, 0; green, 0; blue, 0 }  ][line width=0.08]  [draw opacity=0] (10.72,-5.15) -- (0,0) -- (10.72,5.15) -- (7.12,0) -- cycle    ;
\draw  [dash pattern={on 1.5pt off 1.5pt}]  (210,279) -- (210,243) ;
\draw [shift={(210,240)}, rotate = 90] [fill={rgb, 255:red, 0; green, 0; blue, 0 }  ][line width=0.08]  [draw opacity=0] (10.72,-5.15) -- (0,0) -- (10.72,5.15) -- (7.12,0) -- cycle    ;
\draw  [dash pattern={on 1.5pt off 1.5pt}]  (344,305) -- (400,304.05) ;
\draw [shift={(403,304)}, rotate = 179.03] [fill={rgb, 255:red, 0; green, 0; blue, 0 }  ][line width=0.08]  [draw opacity=0] (10.72,-5.15) -- (0,0) -- (10.72,5.15) -- (7.12,0) -- cycle    ;
\draw  [dash pattern={on 1.5pt off 1.5pt}]  (294,305) -- (238,304.05) ;
\draw [shift={(235,304)}, rotate = 0.97] [fill={rgb, 255:red, 0; green, 0; blue, 0 }  ][line width=0.08]  [draw opacity=0] (10.72,-5.15) -- (0,0) -- (10.72,5.15) -- (7.12,0) -- cycle    ;
\draw  [dash pattern={on 1.5pt off 1.5pt}]  (100,240) -- (100.93,277) ;
\draw [shift={(101,280)}, rotate = 268.57] [fill={rgb, 255:red, 0; green, 0; blue, 0 }  ][line width=0.08]  [draw opacity=0] (10.72,-5.15) -- (0,0) -- (10.72,5.15) -- (7.12,0) -- cycle    ;
\draw   (501,303) .. controls (501,289.19) and (512.19,278) .. (526,278) .. controls (539.81,278) and (551,289.19) .. (551,303) .. controls (551,316.81) and (539.81,328) .. (526,328) .. controls (512.19,328) and (501,316.81) .. (501,303) -- cycle ;
\draw  [dash pattern={on 1.5pt off 1.5pt}]  (526,278) .. controls (510.65,245.66) and (494.97,211.4) .. (454.51,216.63) ;
\draw [shift={(452,217)}, rotate = 350.61] [fill={rgb, 255:red, 0; green, 0; blue, 0 }  ][line width=0.08]  [draw opacity=0] (10.72,-5.15) -- (0,0) -- (10.72,5.15) -- (7.12,0) -- cycle    ;
\draw  [dash pattern={on 1.5pt off 1.5pt}]  (526,278) .. controls (566,248) and (511.33,137) .. (229.33,199) ;
\draw [shift={(229.33,199)}, rotate = 347.6] [fill={rgb, 255:red, 0; green, 0; blue, 0 }  ][line width=0.08]  [draw opacity=0] (10.72,-5.15) -- (0,0) -- (10.72,5.15) -- (7.12,0) -- cycle    ;

\draw (100.22,212.69) node  [font=\Large,xscale=0.5,yscale=0.5]  {$c$};
\draw (210.66,302.69) node  [font=\Large,xscale=0.5,yscale=0.5]  {$b$};
\draw (320.73,303.69) node  [font=\Large,xscale=0.5,yscale=0.5]  {$f$};
\draw (429.11,302.69) node  [font=\Large,xscale=0.5,yscale=0.5]  {$g$};
\draw (317.22,213.69) node  [font=\Large,xscale=0.5,yscale=0.5]  {$e$};
\draw (101.43,303.69) node  [font=\Large,xscale=0.5,yscale=0.5]  {$a$};
\draw (211.11,213.69) node  [font=\Large,xscale=0.5,yscale=0.5]  {$d$};
\draw (428.11,214.69) node  [font=\Large,xscale=0.5,yscale=0.5]  {$h$};
\draw (526,303) node  [font=\Large,xscale=0.5,yscale=0.5]  {$w$};
\end{tikzpicture}
\caption{The AtIAF $I_{att}$ transformed from $I_{ex}$ by Proposition \ref{att-reduce}}
\label{fig2}
\vspace{-5mm}
\end {figure}

    Except the uncertain attacks inside $S$, in the next we will give tractable methods to completely find all of the relevant attacks under $\ad$,$\sta$ and $\co$, and we will show that both the addition and removal variants of \relevance\ problem under these three semantics are all in P complexity class.

    Firstly we discuss $\ad$ and $\sta$, under which verification of a set of arguments $S$ is not influenced by those attacks whose neither attacker nor target is included in $S$. Hence, any uncertain attacks not related to $S$ are irrelevant for verification of $S$ under these two semantics and the relevant attacks are related to $S$. 

    Under $\ad$ semantics, an uncertain attack $r=(a,b)$ not inside $S$ can become $\true$-relevant in two cases: 1. $a\in S$ and $b$ is a possible (necessary) attacker of $S$ which is not attacked back by $S$, implying that $r$ needs to be added when it is confirmed that $b$ attacks $S$ and no one else in $S$ attacks $b$; and 2. $b\in S$ and $a$ is not attacked by $S$, implying that $r$ is unavoidable to be removed when all of the uncertain attacks from $S$ to $a$ are decided to be absent. Meanwhile, $a$ does not attack other arguments in $S$ except $b$, otherwise $r$ turns out to be irrelevant. Therefore, we have the following proposition.
\begin{proposition}\label{ad_rel}
    Given an AtIAF $I=\langle\A,\emptyset,\R,\R^?\rangle$ and a set of argument $S\subseteq\A$ such that $S$ is not stable-$\ad$ w.r.t. $I$, for each attack $r=(a,b)\in\R^?$ s.t. $a\notin S$ or $b\notin S$, \\
    - $r\in RE^+(I,S,\ad$-$\true)$ iff $a\in S, b\notin S_I^+$ and $S\not\subseteq\{b\}^\unatt_I$; and\\
    - $r\in RE^-(I,S,\ad$-$\true)$ iff $a\notin S_I^-, a\notin S_I^+$ and $b\in S$.
\end{proposition}
\begin{proof}
    1. ($\Rightarrow$) Since $r\in RE^+(I,S,\ad$-$\true)$, there is a completion $F$ of $I$ where $r$ is present and $S\in\ad(F)$ whereas removal of $r$ will lead $S$ fail to be an $\ad$-extension. Thus according to the property of $\ad$ semantics, we have that $a\in S$, $b$ attacks $S$ in $F$, hence $S\not\subseteq\{b\}^u_I$ holds. Meanwhile, $a$ is the unique attacker of $b$ in $S$, otherwise $S$ is still an $\ad$-extension after removal of $r$ in $F$, yielding a contradiction. Hence $b\notin S_I^+$ holds.

    ($\Leftarrow$) Since $S$ is not stable-$\ad$, $PosVer_\ad(I,S)=true$ holds, then $S^-_I\cap S^u_I\neq\emptyset$. Let $I'=I+\{(b,s)\in\R^?\mid s\in S\}-\{(s,b)\in\R^?\mid s\in S\}+\{(s,v)\in\R^?\mid s\in S, v\notin S\}-\{(u,s)\in\R^?\mid u\notin S, s\in S\}$ and $I'_r$ be the IAF obtained from making $r$ uncertain from $cert(I')$. We can see that $S$ is attacked by $b$ and attacking back other attackers except $b$ in $I'_r$, thus $S\in\ad(cert(I'_r+\{r\}))$ whereas $S\notin\ad(cert(I'_r-\{r\}))$, which yields $r\in RE^+(I,S,\ad$-$\true)$.

    2. ($\Rightarrow$) Since $r\in RE^-(I,S,\ad$-$\true)$, there is a completion $F$ of $I$ where $r$ is absent and $S\in\ad(F)$ whereas addition of $r$ will lead $S$ fail to be an $\ad$-extension. Thus according to the property of $\ad$ semantics, we have that $a\notin S$, $b\in S$ and $a\notin S^+_F$, yielding $a\notin S^+_I$. Meanwhile, $a\notin S^-_F$, otherwise $S\notin\ad(F)$, yielding contradiction. Hence $a\notin S_I^-$ holds.

    ($\Leftarrow$) Let $I'=I-\{(s,a)\in\R^?\mid s\in S\}+\{(s,v)\in\R^?\mid s\in S, v\notin S\}-\{(u,s)\in\R^?\mid u\notin S, s\in S\}$ and $I'_r$ be the IAF obtained from making $r$ uncertain from $cert(I')$. We can see that $S$ does not attack $a$ but attacking back all of the attackers in $I'_r$. And since $b\in S$, $S\in\ad(cert(I'_r-\{r\}))$ whereas $S\notin\ad(cert(I'_r+\{r\}))$ holds, yielding $r\in RE^-(I,S,\ad$-$\true)$.
\end{proof}
    Under $\sta$ semantics, note that $S$ is stable-$\sta$ only when $S$ is necessary or impossible to attack all of the other arguments, hence only the uncertain attacks from $S$ toward the arguments which have not yet been attacked by $S$ are relevant.
\begin{proposition}\label{st_rel}
    Given an AtIAF $I=\langle\A,\emptyset,\R,\R^?\rangle$ and a set of argument $S\subseteq\A$ such that $S$ is not stable-$\sta$ w.r.t. $I$, for each attack $r=(a,b)\in\R^?$ s.t. $a\notin S$ or $b\notin S$, \\
    - $r\in RE^+(I,S,\sta$-$\true)$ iff $a\in S$ and $b\notin S_{I}^+$; and\\
    - $r\notin RE^-(I,S,\sta$-$\true)$.
\end{proposition}
\begin{proof}
    1. ($\Rightarrow$) Since $r\in RE^+(I,S,\sta$-$\true)$, there is a completion $F$ of $I$ where $r$ is present and $S\in\sta(F)$ whereas removal of $r$ will lead $S$ fail to be a $\sta$-extension. Thus according to the property of $\sta$ semantics, we have that $a$ is the unique attacker of $b$ from $S$ in $F$. Hence $a\in S$ and $b\notin S^+_I$.

    ($\Leftarrow$) Since $S$ is not stable-$\sta$, $PosVer_\sta(I,S)=true$ holds, then $S^u_I\setminus S=\emptyset$. Let $I'=I-\{(s,b)\in\R^?\mid s\in S\}+\{(s,v)\in\R^?\mid s\in S, v\notin S\}$ and $I'_r$ be the IAF obtained from making $r$ uncertain from $cert(I')$. We can see that $S$ attacks all of the arguments outside $S$ except $b$ in $I'_r$, thus $S\in\sta(cert(I'_r+\{r\}))$ whereas $S\notin\sta(cert(I'_r-\{r\}))$ holds, yielding $r\in RE^+(I,S,\sta$-$\true)$.

    2. It is obvious by the property of $\sta$ semantics.
\end{proof}
\begin{example}
    Figure \ref{fig2} illustrates the AtIAF $I_{att}$ transformed from $I_{ex}$ in Figure \ref{fig1} by Proposition \ref{att-reduce}. For the argument set $S=\{a,b,f\}$, we have $RE^+(I_{att},S\cup\{w\},\ad$-$\true)=\{(b,d),(w,d)\}$ and $RE^-(I_{att},S\cup\{w\},\ad$-$\true)=\{(f,b)\}$ according to Proposition \ref{ad_rel}, and $RE^+(I_{att},S\cup\{w\},\sta$-$\true)=\{(b,d),(f,g),(w$ $,h),(w,d)\}$ according to Proposition \ref{st_rel}. Then using Proposition \ref{att-reduce}, we transfer the results to the original input $I_{ex}$ and $S$, and obtain $RE^+(I_{ex},S,\ad$-$\true)=\{(b,d)\}$, $RE^-(I_{ex},S,\ad$-$\true)=\{(f,b),d\}$, $RE^+(I_{ex},S,\sta$-$\true)=\{(b,d)$ $,(f,g)\}$, and $RE^-(I_{ex},S,\sta$-$\true)=\{h,d\}$. The reason why removing $h,d$ is $\sta$-$\true$-relevant lies in that $h,d$ are impossible to be attacked by $S$ once $(b,d)$ is removed, so they must be removed so that $S$ can attack all of the other arguments.      
\end{example}     
    Proposition \ref{ad_rel} and \ref{st_rel} imply that $RE^+(I,S,\sigma$-$\true)$ disjoints with $RE^-(I,S,\sigma$-$\true)$ under $\sigma\in\{\ad,\sta\}$. However, under $\sigma\in\{\co,\gr,\pr\}$, this property does not hold, and one can see a counterexample as follows.
\begin{example}
    Let us consider the IAF $I=\langle\{a,b,c\},\emptyset,\{(b,b),(b,c)\},\{(a,b),(b,a),$ $(a,c)\}\rangle$, $S=\{a\}$ and $\sigma\in\{\co,\gr,\pr\}$. If $(a,c)$ is removed, then $(a,b)$ is needed to be removed for $S$ to be a $\sigma$-extension, otherwise $c\notin S$ will be defended by $S$, hence $(a,b)\in RE^-(I,S,\sigma$-$\true)$. On the other hand, if $(b,a)$ is added, then $(a,b)$ is needed to be added for $S$ to reach $\sigma$-$\true$ status, hence $(a,b)\in RE^+(I,S,\sigma$-$\true)$.
\end{example}
    Now we discuss $\co$ semantics. Before proposing a tractable method for identifying relevant attacks under $\co$, we introduce an algorithm named $OutRel$. With input being an AtIAF $I$, a set of arguments $S$ and an attack $r=(a,b)$ where $b\notin S$ is not necessarily self-attacked, $OutRel(I,S,r)$ determines if there exists a completion $F$ such that $S\in\co(F)$ and $a$ is the unique attacker of $b$ except the attackers attacked by $S$. If $r\in\R^?$ and $OutRel(I,S,r)=true$, then addition of $r$ is $\co$-$\true$-relevant. This is because removing $r$ from $F$ will lead $b$ defended by $S$ and then $S$ fails to be a $\co$-extension. Algorithm \ref{alg_out_rel} shows the details of $OutRel$. At first in line 1, we check whether there are attackers of $b$ which are impossible to be attacked by $S$. If the answer is yes, then it is impossible for $a$ to become the unique attacker of $b$ except $S_F^+$ thus $false$ becomes the output of the algorithm. If the answer is no, we construct the partial completion $I_1$ of $I$ by adding attacks from $S$ to the certain attackers of $b$ 
    (line 2) and then removing all of the uncertain attacks to $b$ except those from $a$ (line 3). Finally, we check whether $S$ is possible to become a $\co$-extension from the IAF $I_1$, whose result is obtained as the output of the algorithm (line 4). It is obvious that the construction of $I_1$ in the algorithm can be done in polynomial time, and $PosVer_\co$ has been proven in $P$ \cite{baumeister2018verification,fazzinga2020revisiting}, hence the complexity of algorithm $OutRel$ is polynomial. 
    
\vspace{-6mm}
\begin{algorithm}\label{alg_out_rel}
\DontPrintSemicolon
  \SetAlgoLined
  \KwIn {An AtIAF $I=\langle\A,\emptyset,\R,\R^?\rangle$, a set of arguments $S\subseteq\A$ and an attack $(a,b)\in\R\cup\R^?$ where $b\notin S$ and $(b,b)\notin\R$}
  \KwOut{Whether there is a completion $F$ of $I$ such that $\{b\}_F^-\setminus S_F^+=\{a\}$ and $S\in\co(F)$}
  \lIf{$(\{b\}_I^-\setminus\{a\})\cap S_I^\unatt\neq\emptyset$}{\Return{false}}
  $I_0\gets I + \{(s,v)\in\R^?\mid s\in S, v\in \{b\}_I^-\setminus\{a\}\}$\;
  $I_1\gets I_0 - \{(u,b)\in\R^?\mid u\neq a\}$\;
  \Return{$PosVer_\co(I_1,S)$}\;
  \caption{$OutRel(I,S,(a,b))$}\label{alg3}
\end{algorithm}
\vspace{-5mm}

    Now we present Theorem \ref{co_rel_theo}, which gives various tractable methods for identifying relevance for three different types of attacks that cover all of the possible attacks except attacks inside $S$: (1) attacks with both attackers and targets are outside $S$, (2) attacks whose targets are inside $S$, and (3) attacks whose attackers are inside $S$. Note that as the same as for $\ad$ and $\sta$ semantics, the attacks whose both attackers and targets belong to $S$ must be removed for satisfying conflict-freeness and then reaching $\co$-$\true$ status, so removing them is $\co$-$\true$-relevant whereas adding is not. 
\begin{lemma}\label{lem_out_rel}
    Given an attIAF $I=\langle\A,\emptyset,\R,\R^?\rangle$ and a set of arguments $S\subseteq\A$ such that $S$ is not stable-$\co$ w.r.t. $I$, for each attack $r=(a,b)\in\R\cup\R^?$ s.t. $b\notin S$ and $(b,b)\notin\R$, there exists a completion $F$ of $I$ such that $\{b\}_F^-\setminus S_F^+=\{a\}$ and $S\in\co(F)$ iff $OutRel(I,S,r)=true$.  
\end{lemma}
\begin{proof}
    ($\Rightarrow$) Assume $(\{b\}_I^-\setminus\{a\})\cap S_I^u\neq\emptyset$, which means that there is some attacker of $b$ except $a$ which is impossible to be attacked by $S$ in $I$, then this will also happen in $F$ and $a$ fails to be the unique attacker of $b$ except arguments of $S^+_F$, a contradiction occurs. Let $\R_F$ be the set of attacks present in $F$ and $F'=\langle\A,\R_F'\rangle$ where $\R_F'=\R_F\cup\{(s,v)\in\R^?\mid s\in S, v\in S^+_F\}\setminus\{(u,b)\in\R^?\mid u\neq a\}$. We can see that $F'$ is also a completion of $I$ where $S\in\co(F')$ still holds according to the property of $\co$ semantics. Meanwhile, $\{(s,v)\in\R^?\mid s\in S, v\in \{b\}_I^-\setminus\{a\}\}\subseteq\R_F'$ and $\{(u,b)\in\R^?\mid u\neq a\}\not\subseteq\R_F'$, thus $F'$ is also a completion of $I_1$ obtained in the process of $OutRel(I,S,r)$. Hence $OutRel(I,S,r)=PosVer_\co(I_1,S)=true$.

    ($\Leftarrow$) Since $OutRel(I,S,r)=true$, there is a completion $F$ of both $I$ and $I_1$ such that $S\in\co(F)$, where $I_1$ is the IAF obtained in the process of $OutRel(I,S,r)$. We can see that in $I_1$, all the possible and necessary attackers of $b$ are attacked by $S$ except $a$. Thus $a$ is the unique possible attacker of $b$ except the arguments attacked by $S$ in any completion of $I_1$. And since $S\in\co(F)$, $\{b\}_F^-\setminus S_F^+=\{a\}$, otherwise $\{b\}_F^-\setminus S_F^+=\emptyset$ is derived, then $b\notin S$ will be defended by $S$, a contradiction occurs. Hence $F$ satisfies the condition, which completes the proof.
\end{proof}
\begin{theorem}\label{co_rel_theo} 
    Given an AtIAF $I=\langle\A,\emptyset,\R,\R^?\rangle$ and a set of arguments $S\subseteq\A$ such that $S$ is not stable-$\co$ w.r.t. $I$, for each attack $r=(a,b)\in\R^?$ s.t. $a\notin S$ or $b\notin S$, the following holds.\\
    (1) If $a\notin S$ and $b\notin S$, \\
    - $r\in RE^+(I,S,\co$-$\true)$ iff $(b,b)\notin\R$ and $OutRel(I,S,r)=true$; and\\
    - $r\notin RE^-(I,S,\co$-$\true)$.\\
    (2) If $a\notin S$ and $b\in S$,  \\
    - $r\notin RE^+(I,S,\co$-$\true)$;  and\\
    - $r\in RE^-(I,S,\co$-$\true)$ iff $a\notin S_I^+$ and $PosVer_\co(I-\{(s,a)\in\R^?\mid s\in S\},S)=true$. \\
    (3) If $a\in S$ and $b\notin S$, \\
    - $r\in RE^+(I,S,\co$-$\true)$ iff $(b,b)\notin\R$ and $OutRel(I,S,r)=true$, or $S\not\subseteq \{b\}^\unatt_I$, $b\notin S^+_I$ and $PosVer_\co(I+\{(a,b)\},S)=true$; and\\
    - $r\in RE^-(I,S,\co$-$\true)$ iff $\exists v\in\A\setminus(S\cup\{b\})$ s.t. $(b,v)\in\R\cup\R^?$, $(v,v)\notin\R$ and $OutRel(I,S,(b,v))=true$.
\end{theorem}
\begin{proof}
    First, Lemma \ref{lem_out_rel} actually is about the soundness and completeness of $OutRel$ algorithm. Then based on Lemma \ref{lem_out_rel}, we prove this theorem.
    
    (1) If $a\notin S$ and $b\notin S$,
    
    1. $(\Rightarrow)$ Since $r\in RE^+(I,S,\co$-$\true)$, there is a completion $F$ of $I$ where $r$ is present, $a,b\notin S$ and $S\in\co(F)$ whereas removal of $r$ will lead $S$ fail to be a $\co$-extension. This happens only when removal of $r$ will lead $b$ defended by $S$ in $F$, thus $(b,b)\notin\R$ and $\{b\}_F^-\setminus S_F^+=\{a\}$, otherwise $b$ can not be defended by $S$ after removal of $r$, yielding a contradiction. Therefore, $OutRel(I,S,r)=true$ holds by Lemma \ref{lem_out_rel}.
    
    $(\Leftarrow)$ $OutRel(I,S,r)=true$ means that there is a completion $F$ of $I$ that $\{b\}_F^-\setminus S_F^+=\{a\}$ and $S\in\co(F)$ by Lemma \ref{lem_out_rel}. We can see that removal of $(a,b)$ will lead $b$ defended by $S$ and then $S$ fail to be a $\co$-extension in $F$, thus $r\in RE^+(I,S,\co$-$\true)$.

    2. It is obvious by the property of $\co$ semantics.\\
    
    (2) If $a\notin S$ and $b\in S$,
    
    1. It is obvious by the property of $\co$ semantics. 

    2. $(\Rightarrow)$ Since $r\in RE^-(I,S,\co$-$\true)$, there is a completion $F$ of $I$ where $r$ is absent, $a\notin S, b\in S$ and $S\in\co(F)$ whereas addition of $r$ will lead $S$ fail to be a $\co$-extension. This happens only when addition of $r$ will lead $S$ fail to be admissible in $F$, which implies that $a\notin S^+_F$, then yielding $a\notin S^+_I$ and that $F$ is also the completion of $I-\{(s,a)\in\R^?\mid s\in S\}$. And since $S\in\co(F)$, $PosVer_\co(I-\{(s,a)\in\R^?\mid s\in S\},S)=true$ holds.

    $(\Leftarrow)$ Since $a\notin S_I^+$, any completion of $I'=I-\{(s,a)\in\R^?\mid s\in S\}$ satisfies that $S$ does not attack $a$. Since $PosVer_\co(I',S)=true$, there is a completion $F$ of both $I'$ and $I$ that $S\in\co(F)$ and $a\notin S^+_F$, thus addition of $r$ will lead $S$ fail to be a $\co$-extension in $F$, yielding $r\in RE^-(I,S,\co$-$\true)$.\\

    (3) If $a\in S$ and $b\notin S$,

    1. $(\Rightarrow)$ Since $r\in RE^+(I,S,\co$-$\true)$, there is a completion $F$ of $I$ where $r$ is present, $a\in S,b\notin S$ and $S\in\co(F)$ whereas removal of $r$ will lead $S$ fail to be a $\co$-extension. This happens in two cases: 1. removal of $r$ causes that $b$ is defended by $S$, thus $(b,b)\notin\R$ and $OutRel(I,S,r)=true$ holds for the same reason as when $a,b\notin S$; and 2. removal of $r$ causes that $S$ fails to be admissible, i.e., $S$ is attacked by $b$ but $r$ is the unique attack from $S$ to $b$. Thus $S\not\subseteq \{b\}^\unatt_I$ and $b\notin S^+_I$ holds. And we can see that $F$ is also a completion of $I+\{(a,b)\}$, hence $PosVer_\co(I+\{(a,b)\},S)=true$. 

    $(\Leftarrow)$ If $(b,b)\notin\R$ and $OutRel(I,S,r)=true$, then $r\in RE^+(I,S,\co$-$\true)$ for the same reason as when $a,b\notin S$. If $S\not\subseteq \{b\}^\unatt_I$, $b\notin S^+_I$ and $PosVer_\co(I+\{(a,b)\},S)=true$, assume that $F=\langle\A,\R_F\rangle$ is one of completions of $I+\{(a,b)\}$ that $S\in\co(F)$. Let $F'=\langle\A,\R_F\cup\{(b,s)\in\R^?\mid s\in S\}\setminus\{(s,b)\in\R^?\mid s\in S, s\neq a\}\rangle$. We can see that $b$ attacks $S$ while $a$ becomes the unique attacker of $b$ from $S$ and $F'$ is also the completion of $I$ since $b\notin S^+_I$. Meanwhile, $S\in\co(F')$ due to the property of $\co$ semantics. Therefore, there is a completion $F'$ of $I$ that $S\in\co(F')$ whereas removal of $r$ will lead $S$ fail to be a $\co$-extension, hence $r\in RE^+(I,S,\co$-$\true)$.

    2. $(\Rightarrow)$ Since $r\in RE^-(I,S,\co$-$\true)$, there is a completion $F$ of $I$ where $r$ is absent, $a\in S,b\notin S$ and $S\in\co(F)$ whereas addition of $r$ will lead $S$ fail to be a $\co$-extension. This happens only when addition of $r$ will lead some new argument $v\neq b$ attacked by $b$ defended by $S$, hence $(b,v)\in\R\cup\R^?$ and $(v,v)\notin\R$, otherwise $v$ cannot be defended by $S$ in any completion, a contradiction occurs. Meanwhile, this means that $b$ is the unique attacker of $v$ except the arguments attacked by $S$ in $F$, thus $OutRel(I,S,(b,v))=true$ holds by Lemma \ref{lem_out_rel}. In conclusion, there exist an argument $v$ satisfying the condition, which completes the proof.

    $(\Leftarrow)$ Let $v\in\A\setminus(S\cup\{b\})$ be the argument that $OutRel(I,S,(b,v))=true$. Then there is a completion $F$ of $I$ that $\{v\}_F^-\setminus S_F^+=\{b\}$ and $S\in\co(F)$ by Lemma \ref{lem_out_rel}, thus $r$ is absent in $F$, otherwise all of the attackers of $v$ including $b$ are attacked by $S$ and then $v\notin S$ is defended by $S$, a contradiction occurs. Meanwhile, addition of $r$ will lead $v$ defended by $S$ and $S$ fail to be a $\co$-extension, hence $r\in RE^-(I,S,\co$-$\true)$.
\end{proof}

\begin{example}
    Consider the IAF $I_{att}$ in Figure \ref{fig2} and the set of arguments $S=\{a,b,w\}$. We can see that the $\co$-$\true$-relevant attacks for $S$ are $(f,g),(f,b),(w,h)$ $,(b,d),(w,d)$ and $(b,e)$, together covering all of the three conclusions listed in Theorem \ref{co_rel_theo} when an uncertain element can be relevant under $\co$: (1) $(f,g)\in RE^+$ since $OutRel(I_{att},S,(f,g)))$ $=true$; (2) $(f,b)\in RE^-$ since $f\notin S_{I_{att}}^+$ and $PosVer_\co(I_{att},$ $S)=true$; (3) $(w,h)\in RE^+$ since $OutRel(I_{att},S,(w,h))=true$; $(b,d),(w,d)\in RE^+$ since $d$ attacks $S$ and $S$ is possible to be a $\co$-extension after addition of $(b,d)$ or $(w,d)$; and $(b,e)\in RE^-$ since $OutRel(I_{att},S,(e,f))=true$.
\end{example}
    By Proposition \ref{ad_rel}, \ref{st_rel} and Corollary \ref{att-reduce-cor}, we can conclude that deciding \relevance\ for verification under $\ad$ and $\co$ semantics is in \textit{P}. In Theorem \ref{co_rel_theo}, since $PosVer$, $OutRel$ and other operations can all be done in polynomial time, deciding \relevance\ under $\co$ is also in \textit{P}.
\begin{corollary}
    For $j\in\{\ad,\sta,\co\}\times\{\true,\false\}$, $j$-\relevance\ is in P.
\end{corollary}
    Finally, we give a method for proving $\Sigma_2^p$ lower bound for $\pr$-\relevance, by giving a reduction from $PosVer$ problem w.r.t. AtIAFs under $\pr$, which is proven to be $\Sigma_2^p$-$c$ in \cite{baumeister2018verification}.
\begin{proposition}
    Given an AtIAF $I=\langle\A,\emptyset,\R,\R^?\rangle$ and a set of arguments $S\subseteq\A$, let $I'=\langle\A\cup\{w_1\},\{w_2\},\R\cup\{(w_2,s)\mid s\in S\},\R^?\cup\{(w_1,w_2)\}$ where $w_1,w_2\notin\A$. \\
    - $PosVer_\pr(I,S)=true$ iff $(w_1,w_2)\in RE^+(I'+\{w_2\},S\cup\{w_1\},\pr$-$\true)$; and\\
    - $PosVer_\pr(I,S)=true$ iff $w_2\in RE^-(I'-\{(w_1,w_2)\},S\cup\{w_1\},\pr$-$\true)$.
\end{proposition}
\begin{proof}
    1. $(\Rightarrow)$ Since $PosVer_\pr(I,S)=true$, there is a completion $F=\langle\A_F,\R_F\rangle$ of $I$ such that $S\in\pr(F)$. According to the construction of $I'$ and the property of $\pr$ semantics, it is easy to see that $F'=\langle\A_F\cup\{w_1,w_2\},\R_F\cup\{(w_2,s)\mid s\in S\}\cup\{(w_1,w_2\}\rangle$ is also a completion of $I'+\{w_2\}$ and $S\cup\{w_1\}\in\pr(F')$. However, removal of $(w_1,w_2)$ in $F'$ will lead $S\cup\{w_1\}$ not admissible thus fail to be a $\pr$-extension, thus $(w_1,w_2)\in RE^+(I'+\{w_2\},S\cup\{w_1\},\pr$-$\true)$ holds.

    $(\Leftarrow)$ $(w_1,w_2)\in RE^+(I'+\{w_2\},S\cup\{w_1\},\pr$-$\true)$ implies that there is a completion $F=\langle\A_F,\R_F\rangle$ of $I'+\{w_2\}$ where $S\cup\{w_1\}\in\pr(F)$ and $(w_1,w_2)$ is present, thus $S\in\pr(F')$ holds where $F'=F_{\downarrow\A_F\setminus\{w_1,w_2\}}$ by the construction of $I'$. Obviously $F'$ is also a completion of $I$, hence $PosVer_\pr(I,S)=true$ holds.

    2. $(\Rightarrow)$ Since $PosVer_\pr(I,S)=true$, there is a completion $F=\langle\A_F,\R_F\rangle$ of $I$ such that $S\in\pr(F)$. According to the construction of $I'$ and the property of $\pr$ semantics, it is easy to see that $I_F=\langle\A_F\cup\{w_1\},\{w_2\},\R_F\cup\{(w_2,s)\mid s\in S\},\emptyset\rangle$ is a partial completion of $I'$, while $S\cup\{w_1\}\in\pr(cert(I_F-\{w_2\}))$ and $S\cup\{w_1\}\notin\pr(cert(I_F+\{w_2\}))$ hold. Hence $w_2\in RE^-(I'-\{(w_1,w_2)\},S\cup\{w_1\},\pr$-$\true)$.

    $(\Leftarrow)$ $w_2\in RE^-(I'-\{(w_1,w_2)\},S\cup\{w_1\},\pr$-$\true)$ implies that there is a completion $F=\langle\A_F,\R_F\rangle$ of $I'-\{(w_1,w_2)\}$ such that $S\cup\{w_1\}\in\pr(F)$ and $w_2$ is present in $F$, thus $S\in\pr(F')$ where $F'=F_{\downarrow\A_F\setminus\{w_1\}}$ by the construction of $I'$. Obviously $F'$ is also a completion of $I$, hence $PosVer_\pr(I,S)=true$ holds.
\end{proof}

\section{Strong Relevance for Stability of Verification}
    In this section we introduce the notion of strong relevance, which characterizes the necessity of resolution of relevant elements for reaching stability. Also, we define the related decision problem and give precise complexity results for each common semantics.  
\subsection{Definition of strong relevance}    
    For a set of arguments where the $j$-stability can be reached, it is interesting to explore which elements are necessary to be added or removed in all partial completions where $j$-stability holds. Such elements are defined as follows to be {\it strongly $j$-relevant} for $S$ w.r.t. the IAF.
\begin{definition}[Strong relevance for stability of a set of arguments]
    Given an IAF $I=\langle\A,\A^?,\R,\R^?\rangle$, a set of arguments $S\subseteq\A\cup\A^?$, a verification status $j$ satisfying that there is at least one $I_p\in part(I)$ such that $S$ is stable-$j$ w.r.t. $I_p$, and an uncertain element $e\in\A^?\cup\R^?$, addition (resp., removal) of $e$ is strongly $j$-relevant for $S$ w.r.t. $I$ iff for each $I'\in part(I-\{e\})$(resp., $part(I+\{e\})$), $S$ is not stable-$j$ w.r.t. $I'$. We use $SRE^+(I,S,j)$ and $SRE^-(I,S,j)$ to respectively denote the uncertain elements whose addition and removal are strongly $j$-relevant for $S$ w.r.t. $I$.
\end{definition}
    Obviously for each uncertain attack inside $S$, its removal is strongly -$\true$-relevant since all the semantics we discuss require conflict-freeness. And for each uncertain argument belonging to $S$, its addition must be strongly -$\true$-relevant since every argument of $S$ should be present so that $S$ can be an extension.
    
    Figure \ref{rel_str_rel_fig} illustrates the inclusion-relationships of various relevant and strongly relevant sets of elements. It is obvious that a strongly relevant element is also relevant, i.e., $SRE^+(I,S,j)\subseteq RE^+(I,S,j)$ and $SRE^-(I,S,j)\subseteq RE^-(I,S,j)$. However, differently from relevance, -$\true$ strong relevance of an action within \{addition, removal\} does not always coincide with -$\false$ strong relevance of its opposite action. We give such instances in the subsequent example. Meanwhile, if addition of an element $e$ is strongly $\sigma$-$\true$(resp., $\false$)-relevant, then it is trivial that removal of $e$ is not $\sigma$-$\true$(resp., $\false$)-relevant.  
\begin{example}
    Continue with the IAF $I_{ex}$ and the set of arguments $S=\{a,b\}$ in Figure \ref{fig1}. We have $SRE^+(I_{ex},S,\co)=\{(f,g)\}$ and $SRE^-(I_{ex},S,\co)=\{h,(b,e),(f,b)\}$. For $\sigma\in\{\ad,\co,\pr\}$, removal of $(f,b)$ is strongly $\sigma$-$\true$-relevant since it is impossible for $S$ to be a $\sigma$-extension once $(f,b)$ is present. Although being $\sigma$-$\true$-relevant, removal of $d$ is not strongly $\sigma$-$\true$-relevant, since when $(b,d)$ is added, $d$ will no longer influence the verification of $S$, thus addition of $d$ will not lead $S$ to fail to be a $\sigma$-extension. One can see that $(f,b)\in SRE^-(I_{ex},S,\ad$-$\true)$ but $\notin SRE^+(I_{ex},S,\ad$-$\false)$. Moreover, $(f,g)\in SRE^+(I_{ex},S,\co$-$\true)$ but $\notin SRE^-(I_{ex},S,\co$-$\false)$. Therefore, addition and removal action do not hold the dual property in terms of strong relevance.  
\end{example}
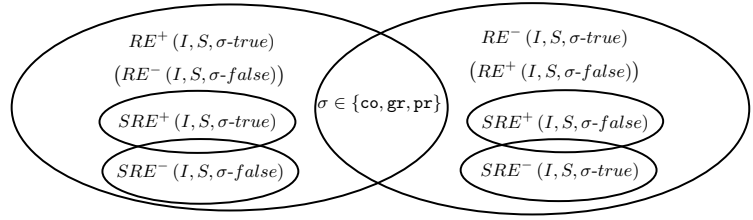
\begin{figure}[t]
\centering 
\tikzset{every picture/.style={line width=0.75pt}} 

\begin{tikzpicture}[x=0.5pt,y=0.5pt,yscale=-1,xscale=1]

\draw   (0.33,142) .. controls (0.33,98.92) and (74.18,64) .. (165.27,64) .. controls (256.36,64) and (330.2,98.92) .. (330.2,142) .. controls (330.2,185.08) and (256.36,220) .. (165.27,220) .. controls (74.18,220) and (0.33,185.08) .. (0.33,142) -- cycle ;
\draw   (66.8,152.9) .. controls (66.8,139.92) and (99.98,129.4) .. (140.9,129.4) .. controls (181.82,129.4) and (215,139.92) .. (215,152.9) .. controls (215,165.88) and (181.82,176.4) .. (140.9,176.4) .. controls (99.98,176.4) and (66.8,165.88) .. (66.8,152.9) -- cycle ;
\draw   (230,143) .. controls (230,98.82) and (303.5,63) .. (394.17,63) .. controls (484.83,63) and (558.33,98.82) .. (558.33,143) .. controls (558.33,187.18) and (484.83,223) .. (394.17,223) .. controls (303.5,223) and (230,187.18) .. (230,143) -- cycle ;
\draw   (69,190.4) .. controls (69,176.87) and (102.18,165.9) .. (143.1,165.9) .. controls (184.02,165.9) and (217.2,176.87) .. (217.2,190.4) .. controls (217.2,203.93) and (184.02,214.9) .. (143.1,214.9) .. controls (102.18,214.9) and (69,203.93) .. (69,190.4) -- cycle ;
\draw   (345,152.4) .. controls (345,139.42) and (378.18,128.9) .. (419.1,128.9) .. controls (460.02,128.9) and (493.2,139.42) .. (493.2,152.4) .. controls (493.2,165.38) and (460.02,175.9) .. (419.1,175.9) .. controls (378.18,175.9) and (345,165.38) .. (345,152.4) -- cycle ;
\draw   (340,189.45) .. controls (340,176.5) and (373.18,166) .. (414.1,166) .. controls (455.02,166) and (488.2,176.5) .. (488.2,189.45) .. controls (488.2,202.4) and (455.02,212.9) .. (414.1,212.9) .. controls (373.18,212.9) and (340,202.4) .. (340,189.45) -- cycle ;

\draw (140.9,152.9) node  [xscale=0.75,yscale=0.75]  {$SRE^{+}\left( I,S,\sigma \text{\mbox{-}} true\right)$};
\draw (143.1,190.4) node  [xscale=0.75,yscale=0.75]  {$SRE^{-}\left( I,S,\sigma \text{\mbox{-}} false\right)$};
\draw (419.1,153.4) node  [xscale=0.75,yscale=0.75]  {$SRE^{+}\left( I,S,\sigma \text{\mbox{-}} false\right)$};
\draw (415.1,189.45) node  [xscale=0.75,yscale=0.75]  {$SRE^{-}\left( I,S,\sigma \text{\mbox{-}} true\right)$};
\draw (142.2,117.9) node  [xscale=0.75,yscale=0.75]  {$\left( RE^{-}\left( I,S,\sigma \text{\mbox{-}} false\right)\right)$};
\draw (410.94,115.2) node  [xscale=0.75,yscale=0.75]  {$\left( RE^{+}\left( I,S,\sigma \text{\mbox{-}} false\right)\right)$};
\draw (411.2,88.9) node  [xscale=0.75,yscale=0.75]  {$RE^{-}\left( I,S,\sigma \text{\mbox{-}} true\right)$};
\draw (144.2,91.9) node  [xscale=0.75,yscale=0.75]  {$RE^{+}\left( I,S,\sigma \text{\mbox{-}} true\right)$};
\draw (230,130.4) node [anchor=north west][inner sep=0.75pt]  [xscale=0.75,yscale=0.75]  {$\sigma\in\{\co,\gr,\pr\}$};
\end{tikzpicture}
\caption{Inclusion-relationships of sets of relevant and strong relevant elements, where an overlap of two ellipses means that the related sets may intersect under some semantics, and one ellipse within another indicates a subset under all semantics.}
\label{rel_str_rel_fig}
\vspace{-5mm}
\end {figure}
\subsection{Complexity analysis of strong relevance} 
Now we discuss the computational complexity for the decision problem of identifying strong relevance, which is formulated as follows. 
\begin{flushleft}
\begin{tabular}{p{1\textwidth}}
\toprule
$j$-\srelevance\ of action $\textbf{a}$ \\
\hline
\textbf{Given:}\ \ An IAF $I=\langle\A,\A^?,\R,\R^?\rangle$, a set of arguments $S\subseteq\A\cup\A^?$, a verification status $j$ s.t. there is at least one $I'\in part(I)$ such that $S$ is stable-$j$ w.r.t. $I'$, an action $\textbf{a}\in\{$addition,removal$\}$, and an uncertain element $\textbf{e}\in\A^?\cup\R^?$. \\
\textbf{Question:}\ \ Is \textbf{a} of $\textbf{e}$ strongly $j$-relevant for $S$ w.r.t. $I$? \\
\bottomrule
\end{tabular}
\vspace{-3mm}
\end{flushleft}

    We start with a proposition giving a reduction from \srelevance\ problem to $PosVer$ and $NecVer$ problem. Directly according to the definition of stability, $S$ can reach $\sigma$-$\true$ stability only when $PosVer_\sigma(I,S)=true$, while $S$ can reach $\sigma$-$\false$ stability only when $NecVer_\sigma(I,S)=false$. Therefore, addition of an uncertain element $e$ is strongly $\sigma$-$\true$-relevant iff $S$ is impossible to become $\sigma$-$\true$ after the removal of $e$ in $I$, i.e., $PosVer_\sigma(I-\{e\},S)=false$.  
\begin{proposition}\label{srel_alt}
    Given an IAF $I=\langle\A,\A^?,\R,\R^?\rangle$, a set of arguments $S\subseteq\A\cup\A^?$ and semantics $\sigma$, if $PosVer_\sigma(I,S)=true$ (resp., $NecVer_\sigma(I,S)=false$), then for each uncertain element $e\in\A^?\cup\R^?$, \\
    - $e\in SRE^+(I,S,\sigma$-$\true)$ (resp., $e\in SRE^+(I,S,\sigma$-$\false)$) iff $PosVer_\sigma(I-\{e\},S)=false$ (resp., $NecVer_\sigma(I-\{e\},S)=true$); and\\
    - $e\in SRE^-(I,S,\sigma$-$\true)$ (resp., $e\in SRE^-(I,S,\sigma$-$\false)$) iff $PosVer_\sigma(I+\{e\},S)=false$ (resp., $NecVer_\sigma(I+\{e\},S)=true$).
\end{proposition}
\begin{proof}
    The proof can be done directly according to the definition of $PosVer$, $NecVer$ and strong relevance.
\end{proof}
    Using the complexity results of $PosVer$ and $NecVer$ in \cite{baumeister2018verification,fazzinga2020revisiting}, we directly obtain the upper bounds of \srelevance\ according to Proposition \ref{srel_alt}.
\begin{corollary}
    The following results hold:\\
    1. for $j\in\{\ad,\sta,\co,\gr\}\times\{\true,\false\}$, $j$-\srelevance\ is in $P$;\\
    2. $\pr$-$\true$-\srelevance\ is in $\Pi_2^p$; and\\
    3. $\pr$-$\false$-\srelevance\ is in $coNP$.
\end{corollary}
    In order to prove $coNP$ completeness of $\pr$-$\false$-\srelevance, we reduce from $NecVer$ problem w.r.t. AtIAFs under $\pr$, which has been proven to be $coNP$-$c$ in \cite{baumeister2018verification}.
\begin{proposition}
    Given an AtIAF $I=\langle\A,\emptyset,\R,\R^?\rangle$ and a set of arguments $S\subseteq\A$, let IAF $I'=\langle\A,\{w\},\R,\R^?\cup\{(w,w)\}\rangle$ where $w\notin\A$. The following results hold: \\
    - $NecVer_\pr(I,S)=true$ iff $w\in SRE^+(I'-\{(w,w)\},S,\pr$-$\false)$; and\\
    - $NecVer_\pr(I,S)=true$ iff $(w,w)\in SRE^-(I'+\{w\},S,\pr$-$\false)$.
\end{proposition}
\begin{proof}
    1. Consider the completion $F=cert(I'-\{(w,w)\}+\{w\})$ of $I'-\{(w,w)\}$ where $w\notin S$ and $S$ is $\pr$-$\false$ in $F$ since $S$ does not contain $w$. Thus there is at least one stable-$\pr$-$\false$ partial completion for $S$ w.r.t. $I'-\{(w,w)\}$. 
    
    $(\Rightarrow)$ Assume $w\notin SRE^+(I'-\{(w,w)\},S,\pr$-$\false)$, which implies that there is a completion $F$ of $I'-\{(w,w)\}-\{w\}$ such that $S\notin\pr(F)$. By the construction of $I'$, $F$ is also a completion of $I$, which contradicts $NecVer_\pr(I,S)=true$.

    ($\Leftarrow$) We have $NecVer_\pr(I'-\{(w,w)\}-\{w\},S)=true$. Since $I'-\{(w,w)\}-\{w\}=I$, $NecVer_\pr(I,S)=true$ holds.

    2. Consider the completion $F=cert(I'+\{w\}-\{(w,w)\})$ of $I'+\{w\}$ where $w\notin S$ and $S$ is $\pr$-$\false$ in $F$ since $S$ does not contain $w$. Thus there is at least one stable-$\pr$-$\false$ partial completion for $S$ w.r.t. $I'+\{w\}$. 
    
    $(\Rightarrow)$ Assume $(w,w)\notin SRE^-(I'+\{w\},S,\pr$-$\false)$, which implies that there is a completion $F$ of $I'+\{w\}+\{(w,w)\}$ such that $S\notin\pr(F)$. Due to the property of $\pr$ semantics and the construction of $I'$, we can see that $S\notin\pr(F_{\downarrow\A})$ holds and $F_{\downarrow\A}$ is also a completion of $I$, which contradicts $NecVer_\pr(I,S)=true$.

    ($\Leftarrow$) Assume $NecVer_\pr(I,S)=false$, i.e., there is a completion of $F=\langle\A,\R_F\rangle$ such that $S\notin\pr(F)$. Let $F'=\langle\A\cup\{w\},\R_F\cup\{(w,w)\}\rangle$. Then it is easy to see that $S\notin\pr(F')$ due to the property of $\pr$ semantics and $F'$ is also a completion of $I'+\{w\}+\{(w,w)\}$, hence $NecVer_\pr(I'+\{w\}+\{(w,w)\},S)=false$. Since $(w,w)\in SRE^-(I'+\{w\},S,\pr$-$\false)$, we have $NecVer_\pr(I'+\{w\}+\{(w,w)\},S)=true$ by Proposition \ref{rel_sta_relation}, yielding a contradiction. 
\end{proof}
    Next we show $\Pi_2^p$ completeness of $\pr$-$\true$-\srelevance\ by reducing from $\Pi_2SAT$ problem. Given a pair of disjoint boolean variable sets $X,Y$ and a formula $\varphi$ in conjunctive normal form, $\Pi_2SAT$ asks whether for any assignment $\tau_X$ for $X$, there always exists an assignment $\tau_Y$ for $Y$ such that $\varphi[\tau_X,\tau_Y]=\true$ holds. We translate a $\Pi_2SAT$ instance to an IAF in Definition \ref{sat} based on the translation used in \cite{baumeister2021acceptance}, which is illustrated in Figure \ref{reduce_srel_true}.
\begin{definition}\label{sat}
    Let $(\varphi,X,Y)$ be an instance of $\Pi_2SAT$. And let $\varphi=\wedge_ic_i$ and $c_i=\vee_j\alpha_j$ where the $\alpha_j$ are the literals belonging to $X\cup Y$ that occur in clause $c_i$. We define the IAF $I_\varphi=\langle\A,\A^?,\R,\R^?\rangle$ as follows:\\
    - $\A=\{\overline{x_i}\mid x_i\in X\}\cup\{y_i,\overline{y_i}\in Y\}\cup\{c_i\mid c_i\in\varphi\}\cup\{\varphi,w\}$;\\
    - $\A^?=\{x_i\mid x_i\in X\}$;\\
    - $\R=\{(x_i,\overline{x_i})\mid x_i\in X\}\cup\{(y_i,\overline{y_i}),(\overline{y_i},y_i)\mid y_i\in Y\}\cup\{(x_k,c_i)\mid x_k\in c_i\}\cup\{(\overline{x_k},c_i)\mid \neg x_k\in c_i\}\cup\{(y_k,c_i)\mid y_k\in c_i\}\cup\{(\overline{y_k},c_i)\mid \neg y_k\in c_i\}\cup\{(c_i,\varphi)\mid c_i\in\varphi\}\cup\{(c_i,c_i)\mid c_i\in\varphi\}\cup\{(w,x_i),(w,\overline{x_i})\mid x_i\in X\}\cup\{(w,y_i),(w,\overline{y_i})\mid y_i\in Y\}\cup\{(w,w)\}$; and\\
    - $\R^?=\{(\varphi,\varphi),(\varphi,w)\}$.
\end{definition}
    The following proposition shows the reduction from $\Pi_2SAT$ to $\pr$-$\true$-\srelevance\ problem.
\begin{proposition}
    Given an instance $(\varphi,X,Y)$ of $\Pi_2SAT$, let $I_\varphi$ be the IAF constructed by Definition \ref{sat}. The following results hold: \\
    - $(\varphi,X,Y)\in\Pi_2SAT$ iff $(\varphi,\varphi)\in SRE^+(I_\varphi+\{(\varphi,w)\},\emptyset,\pr$-$\true)$; and\\
    - $(\varphi,X,Y)\in\Pi_2SAT$ iff $(\varphi,w)\in SRE^-(I_\varphi-\{(\varphi,\varphi)\},\emptyset,\pr$-$\true)$.
\end{proposition}
\begin{proof}
    1. Consider the completion $F=cert(I_\varphi+\{(\varphi,w)\}+\{(\varphi,\varphi)\})$. There is no nonempty admissible extension containing $\varphi,w$ and $c_i$ since they are self-attacked. Also, any argument $x_i,\overline{x_i},y_i$ and $\overline{y_i}$ is not contained since they are attacked by $w$ but the unique attacker of $w$ is self-attacked. Therefore, $\emptyset$ is the unique $\pr$-extension of $F$, which means that there is at least one stable-$\pr$-$\true$ partial completion for $\emptyset$ w.r.t. $I_\varphi+\{(\varphi,w)\}$.   
    
    ($\Rightarrow$) Assume $(\varphi,\varphi)\notin SRE^+(I_\varphi+\{(\varphi,w)\},\emptyset,\pr$-$\true)$, i.e., there is a completion $F=\langle\A',\R'\rangle$ of $I_\varphi+\{(\varphi,w)\}-\{(\varphi,\varphi)\}$ such that $\emptyset\in\pr(F)$. Let $\tau_X$ be the assignment for $X$ such that $\tau_X(x_i)=true$ iff $x_i\in\A'$. Since $(\varphi,X,Y)\in\Pi_2SAT$, there is an assignment $\tau_Y$ for $Y$ such that $\varphi[\tau_X,\tau_Y]=\true$ holds. Let $\Omega=\{x_i\mid x_i\in\A'\}\cup\{\overline{x_i}\mid x_i\notin\A'\}\cup\{y_i\mid \tau_Y(y_i)=true\}\cup\{\overline{y_i}\mid \tau_Y(y_i)=false\}\cup\{\varphi\}$. Then we have $\Omega$ is conflict-free and $\{c_i\mid c_i\in\varphi\}\cup\{w\}\subseteq\Omega_F^+$, thus $\Omega\in\sta(F)$ leading $\Omega\in\pr(F)$ to hold, contradicting $\emptyset\in\pr(F)$.

    ($\Leftarrow$) Assume $(\varphi,X,Y)\notin\Pi_2SAT$, i.e., there is an assignment $\tau_X$ for $X$, for any assignment $\tau_Y$ for $Y$, $\varphi[\tau_X,\tau_Y]=\false$ holds. Let $F=\langle\A',\R'\rangle$ be a completion of $I_\varphi+\{(\varphi,w)\}-\{(\varphi,\varphi)\}$ where $x_i\in\A'$ iff $\tau_X(x_i)=true$. Then we can see that there is no non-conflict-free set attacking all of the arguments in $\{c_i\mid c_i\in\varphi\}$, thus $\varphi$ cannot be contained in any admissible extension, yielding the arguments $x_i,\overline{x_i},y_i$ and $\overline{y_i}$ is not contained in any admissible extension in $F$ since they are attacked by $w$ and the unique attacker of $w$ is $\varphi$. Also, the arguments in $\{c_i\mid c_i\in\varphi\}$ and $w$ are not contained since they are self-attacked. Therefore, there is no nonempty admissible extension in $F$ leading $PosVer_\pr(I_\varphi+\{(\varphi,w)\}-\{(\varphi,\varphi)\},\emptyset)=true$ to hold, contradicting $(\varphi,\varphi)\in SRE^+(I_\varphi+\{(\varphi,w)\},\emptyset,\pr$-$\true)$ by Proposition \ref{srel_alt}.

    2. Consider the completion $F=cert(I_\varphi-\{(\varphi,\varphi)\}-\{(\varphi,w)\})$. There is no nonempty admissible extension containing $c_i$ and $w$ since they are self-attacked. Also, any argument $x_i,\overline{x_i},y_i$ and $\overline{y_i}$ is not contained since they are attacked by $w$ but $w$ is not attacked by any other argument. This yields that $\varphi$ is not contained neither since $\varphi$ is attacked by $c_i$ but the attackers of $c_i$ are all not contained in any nonempty admissible extension. Therefore, $\emptyset$ is the unique $\pr$-extension of $F$, which means that there is at least one stable-$\pr$-$\true$ partial completion for $\emptyset$ w.r.t. $I_\varphi-\{(\varphi,\varphi)\}$. 
    
    The proof of both direction is analogous to the above conclusion 1.
\end{proof}
    In conclusion, despite the notion of strong relevance characterizes a stricter type of relevance, the complexity results show that its identification does not become a more difficult problem than identifying relevance.
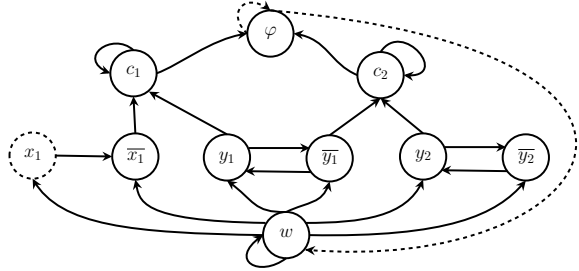
\begin{figure}[t]
\centering 
\tikzset{every picture/.style={line width=0.75pt}} 
\begin{tikzpicture}[x=0.35pt,y=0.35pt,yscale=-1,xscale=1]
\draw   (334,181) .. controls (334,167.19) and (345.19,156) .. (359,156) .. controls (372.81,156) and (384,167.19) .. (384,181) .. controls (384,194.81) and (372.81,206) .. (359,206) .. controls (345.19,206) and (334,194.81) .. (334,181) -- cycle ;
\draw   (188,314) .. controls (188,300.19) and (199.19,289) .. (213,289) .. controls (226.81,289) and (238,300.19) .. (238,314) .. controls (238,327.81) and (226.81,339) .. (213,339) .. controls (199.19,339) and (188,327.81) .. (188,314) -- cycle ;
\draw  [dash pattern={on 1.5pt off 1.5pt}] (77,313) .. controls (77,299.19) and (88.19,288) .. (102,288) .. controls (115.81,288) and (127,299.19) .. (127,313) .. controls (127,326.81) and (115.81,338) .. (102,338) .. controls (88.19,338) and (77,326.81) .. (77,313) -- cycle ;
\draw    (127,313) -- (185,313.95) ;
\draw [shift={(188,314)}, rotate = 180.94] [fill={rgb, 255:red, 0; green, 0; blue, 0 }  ][line width=0.08]  [draw opacity=0] (10.72,-5.15) -- (0,0) -- (10.72,5.15) -- (7.12,0) -- cycle    ;
\draw   (398,316) .. controls (398,302.19) and (409.19,291) .. (423,291) .. controls (436.81,291) and (448,302.19) .. (448,316) .. controls (448,329.81) and (436.81,341) .. (423,341) .. controls (409.19,341) and (398,329.81) .. (398,316) -- cycle ;
\draw   (287,315) .. controls (287,301.19) and (298.19,290) .. (312,290) .. controls (325.81,290) and (337,301.19) .. (337,315) .. controls (337,328.81) and (325.81,340) .. (312,340) .. controls (298.19,340) and (287,328.81) .. (287,315) -- cycle ;
\draw    (335.67,304.83) -- (398.67,304.83) ;
\draw [shift={(401.67,304.83)}, rotate = 180] [fill={rgb, 255:red, 0; green, 0; blue, 0 }  ][line width=0.08]  [draw opacity=0] (10.72,-5.15) -- (0,0) -- (10.72,5.15) -- (7.12,0) -- cycle    ;
\draw    (402.67,330.83) -- (334.67,329.88) ;
\draw [shift={(331.67,329.83)}, rotate = 0.81] [fill={rgb, 255:red, 0; green, 0; blue, 0 }  ][line width=0.08]  [draw opacity=0] (10.72,-5.15) -- (0,0) -- (10.72,5.15) -- (7.12,0) -- cycle    ;
\draw   (610,315) .. controls (610,301.19) and (621.19,290) .. (635,290) .. controls (648.81,290) and (660,301.19) .. (660,315) .. controls (660,328.81) and (648.81,340) .. (635,340) .. controls (621.19,340) and (610,328.81) .. (610,315) -- cycle ;
\draw   (499,314) .. controls (499,300.19) and (510.19,289) .. (524,289) .. controls (537.81,289) and (549,300.19) .. (549,314) .. controls (549,327.81) and (537.81,339) .. (524,339) .. controls (510.19,339) and (499,327.81) .. (499,314) -- cycle ;
\draw    (547.67,303.83) -- (610.67,303.83) ;
\draw [shift={(613.67,303.83)}, rotate = 180] [fill={rgb, 255:red, 0; green, 0; blue, 0 }  ][line width=0.08]  [draw opacity=0] (10.72,-5.15) -- (0,0) -- (10.72,5.15) -- (7.12,0) -- cycle    ;
\draw    (614.67,329.83) -- (546.67,328.88) ;
\draw [shift={(543.67,328.83)}, rotate = 0.81] [fill={rgb, 255:red, 0; green, 0; blue, 0 }  ][line width=0.08]  [draw opacity=0] (10.72,-5.15) -- (0,0) -- (10.72,5.15) -- (7.12,0) -- cycle    ;
\draw   (186,224) .. controls (186,210.19) and (197.19,199) .. (211,199) .. controls (224.81,199) and (236,210.19) .. (236,224) .. controls (236,237.81) and (224.81,249) .. (211,249) .. controls (197.19,249) and (186,237.81) .. (186,224) -- cycle ;
\draw    (351,399) .. controls (208.83,396.87) and (131.99,389.07) .. (103.28,340.26) ;
\draw [shift={(102,338)}, rotate = 61.44] [fill={rgb, 255:red, 0; green, 0; blue, 0 }  ][line width=0.08]  [draw opacity=0] (10.72,-5.15) -- (0,0) -- (10.72,5.15) -- (7.12,0) -- cycle    ;
\draw    (211,199) .. controls (188.13,178.26) and (148.3,207.61) .. (183.74,223.07) ;
\draw [shift={(186,224)}, rotate = 201.09] [fill={rgb, 255:red, 0; green, 0; blue, 0 }  ][line width=0.08]  [draw opacity=0] (10.72,-5.15) -- (0,0) -- (10.72,5.15) -- (7.12,0) -- cycle    ;
\draw   (453,226) .. controls (453,212.19) and (464.19,201) .. (478,201) .. controls (491.81,201) and (503,212.19) .. (503,226) .. controls (503,239.81) and (491.81,251) .. (478,251) .. controls (464.19,251) and (453,239.81) .. (453,226) -- cycle ;
\draw    (478,201) .. controls (502.3,164.39) and (555.05,229.82) .. (505.34,226.21) ;
\draw [shift={(503,226)}, rotate = 6.2] [fill={rgb, 255:red, 0; green, 0; blue, 0 }  ][line width=0.08]  [draw opacity=0] (10.72,-5.15) -- (0,0) -- (10.72,5.15) -- (7.12,0) -- cycle    ;
\draw   (351,399) .. controls (351,385.19) and (362.19,374) .. (376,374) .. controls (389.81,374) and (401,385.19) .. (401,399) .. controls (401,412.81) and (389.81,424) .. (376,424) .. controls (362.19,424) and (351,412.81) .. (351,399) -- cycle ;
\draw    (213,289) -- (211.15,252) ;
\draw [shift={(211,249)}, rotate = 87.14] [fill={rgb, 255:red, 0; green, 0; blue, 0 }  ][line width=0.08]  [draw opacity=0] (10.72,-5.15) -- (0,0) -- (10.72,5.15) -- (7.12,0) -- cycle    ;
\draw    (312,290) -- (230.3,245.27) ;
\draw [shift={(227.67,243.83)}, rotate = 28.7] [fill={rgb, 255:red, 0; green, 0; blue, 0 }  ][line width=0.08]  [draw opacity=0] (10.72,-5.15) -- (0,0) -- (10.72,5.15) -- (7.12,0) -- cycle    ;
\draw    (524,289) -- (480.31,252.91) ;
\draw [shift={(478,251)}, rotate = 39.56] [fill={rgb, 255:red, 0; green, 0; blue, 0 }  ][line width=0.08]  [draw opacity=0] (10.72,-5.15) -- (0,0) -- (10.72,5.15) -- (7.12,0) -- cycle    ;
\draw    (423,291) -- (475.57,252.76) ;
\draw [shift={(478,251)}, rotate = 143.97] [fill={rgb, 255:red, 0; green, 0; blue, 0 }  ][line width=0.08]  [draw opacity=0] (10.72,-5.15) -- (0,0) -- (10.72,5.15) -- (7.12,0) -- cycle    ;
\draw    (353.67,385.83) .. controls (215.83,383.76) and (217.23,361.93) .. (213.57,341.83) ;
\draw [shift={(213,339)}, rotate = 77.37] [fill={rgb, 255:red, 0; green, 0; blue, 0 }  ][line width=0.08]  [draw opacity=0] (10.72,-5.15) -- (0,0) -- (10.72,5.15) -- (7.12,0) -- cycle    ;
\draw    (376,374) .. controls (340.57,374.81) and (341.59,361.52) .. (314.17,341.55) ;
\draw [shift={(312,340)}, rotate = 35.08] [fill={rgb, 255:red, 0; green, 0; blue, 0 }  ][line width=0.08]  [draw opacity=0] (10.72,-5.15) -- (0,0) -- (10.72,5.15) -- (7.12,0) -- cycle    ;
\draw    (401,399) .. controls (551.59,401.78) and (603.59,376.86) .. (633.21,342.14) ;
\draw [shift={(635,340)}, rotate = 129.3] [fill={rgb, 255:red, 0; green, 0; blue, 0 }  ][line width=0.08]  [draw opacity=0] (10.72,-5.15) -- (0,0) -- (10.72,5.15) -- (7.12,0) -- cycle    ;
\draw    (376,374) .. controls (400.64,362.32) and (414.53,365.54) .. (422.09,343.84) ;
\draw [shift={(423,341)}, rotate = 106.45] [fill={rgb, 255:red, 0; green, 0; blue, 0 }  ][line width=0.08]  [draw opacity=0] (10.72,-5.15) -- (0,0) -- (10.72,5.15) -- (7.12,0) -- cycle    ;
\draw    (396.67,385.83) .. controls (458.41,385.83) and (505.74,379.11) .. (522.98,341.35) ;
\draw [shift={(524,339)}, rotate = 112.3] [fill={rgb, 255:red, 0; green, 0; blue, 0 }  ][line width=0.08]  [draw opacity=0] (10.72,-5.15) -- (0,0) -- (10.72,5.15) -- (7.12,0) -- cycle    ;
\draw  [dash pattern={on 1.5pt off 1.5pt}]  (359,156) .. controls (863.67,199.83) and (719.33,442.67) .. (396.67,414.83) ;
\draw [shift={(396.67,414.83)}, rotate = 4.93] [fill={rgb, 255:red, 0; green, 0; blue, 0 }  ][line width=0.08]  [draw opacity=0] (10.72,-5.15) -- (0,0) -- (10.72,5.15) -- (7.12,0) -- cycle    ;
\draw  [dash pattern={on 1.5pt off 1.5pt}]  (334,181) .. controls (311.25,160.36) and (329.69,135.13) .. (356.89,154.43) ;
\draw [shift={(359,156)}, rotate = 218.04] [fill={rgb, 255:red, 0; green, 0; blue, 0 }  ][line width=0.08]  [draw opacity=0] (10.72,-5.15) -- (0,0) -- (10.72,5.15) -- (7.12,0) -- cycle    ;
\draw    (236,224) .. controls (261.02,215.06) and (293.97,187.27) .. (331.13,181.41) ;
\draw [shift={(334,181)}, rotate = 172.81] [fill={rgb, 255:red, 0; green, 0; blue, 0 }  ][line width=0.08]  [draw opacity=0] (10.72,-5.15) -- (0,0) -- (10.72,5.15) -- (7.12,0) -- cycle    ;
\draw    (453,226) .. controls (422.45,213.16) and (421.69,187.17) .. (386.76,181.4) ;
\draw [shift={(384,181)}, rotate = 7.31] [fill={rgb, 255:red, 0; green, 0; blue, 0 }  ][line width=0.08]  [draw opacity=0] (10.72,-5.15) -- (0,0) -- (10.72,5.15) -- (7.12,0) -- cycle    ;

\draw    (376,424) .. controls (347.58,445.56) and (312.44,426.98) .. (348.69,400.62) ;
\draw [shift={(351,399)}, rotate = 145.78] [fill={rgb, 255:red, 0; green, 0; blue, 0 }  ][line width=0.08]  [draw opacity=0] (10.72,-5.15) -- (0,0) -- (10.72,5.15) -- (7.12,0) -- cycle    ;

\draw (359.22,179.69) node  [font=\Large,xscale=0.5,yscale=0.5]  {$\varphi $};
\draw (213.66,312.69) node  [font=\Large,xscale=0.5,yscale=0.5]  {$\overline{x_{1}}$};
\draw (102.43,311.69) node  [font=\Large,xscale=0.5,yscale=0.5]  {$x_{1}$};
\draw (423.66,314.69) node  [font=\Large,xscale=0.5,yscale=0.5]  {$\overline{y_{1}}$};
\draw (312,315) node  [font=\Large,xscale=0.5,yscale=0.5]  {$y_{1}$};
\draw (635.66,313.69) node  [font=\Large,xscale=0.5,yscale=0.5]  {$\overline{y_{2}}$};
\draw (524.43,312.69) node  [font=\Large,xscale=0.5,yscale=0.5]  {$y_{2}$};
\draw (211.22,221.69) node  [font=\Large,xscale=0.5,yscale=0.5]  {$c_{1}$};
\draw (478.22,223.69) node  [font=\Large,xscale=0.5,yscale=0.5]  {$c_{2}$};
\draw (376.22,396.69) node  [font=\Large,xscale=0.5,yscale=0.5]  {$w$};
\end{tikzpicture}
\caption{The IAF created from the $\Pi_2SAT$ instance $((\overline{x_1}\vee y_1)\wedge(\overline{y_1}\vee y_2),\{x_1\},\{y_1,y_2\})$.}
\label{reduce_srel_true}
\vspace{-5mm}
\end {figure}
\section{Related Work and Discussion}
    The notion of relevance is first proposed in \cite{odekerken2023justification} where relevance for stability of justification status of a given argument is studied. We follow that work and study the relevance for stability of verification status of a given set of arguments in this paper. Table \ref{tab1} lists part of complexity results from \cite{odekerken2023justification} and our results about relevance and strong relevance. Compared with that all of the complexities for stability of justification are up to exponential level, we identify tractable methods for solving relevance problems for stability of verification under most considered semantics, showing that they are in \textit{P} complexity level. Particularly, in Theorem \ref{co_rel_theo} for deciding relevance for stability of a given set of arguments $S$ under $\co$-semantics, we divide the uncertain attacks into three separate parts, the attacks towards $S$, from $S$ and not related to $S$, and then give methods accordingly to tackle the relevance problems. Such a dividing way may also benefit solving relevance problem of justification, for instance, in order obtain $\sigma$-$in$-$credulous$($skeptical$)-relevant elements for a given argument $a$, one can divide the uncertain arguments and attacks into two parts, those that are able to reach $a$ (through both certain and uncertain attacks) and those not. Then we can directly conclude that those not able to reach $a$ are irrelevant as long as the semantics discussed satisfies {\it directionality} criterion \cite{baroni2007principle}. Such a dividing intuition is also applied in \cite{liao2011dynamics} for reasoning in dynamic AFs.
    
    Besides the distinction in complexity, the relevance problems discussed in \cite{odekerken2023justification} and this paper hold different properties. Recall that -$\true$-relevance of addition coincides with -$\false$-relevance of removal for verification of a set of arguments. Such a property does not hold for every two of -$in$,-$out$ and -$undec$-$credulous$($skeptical$)-relevance for justification of an argument, because for instance the status `not in' does not coincide with either `out' or `undec'. Lastly, it is easy to see that the notion of strong relevance proposed in this paper can be adapted to justification of an argument. The preliminary results from a first attempt of ours show that the computational complexity of deciding strong relevance tends to be lower than relevance for justification of an argument, which will be elaborated in our future work. This is different from the verification of a set of arguments where relevance and strong relevance hold the same complexity in \textit{P} for the most discussed semantics, as shown in Table \ref{tab1}.
    
    We have not given the precise complexity result for $\gr$ semantics. It turns out to be difficult to find tractable methods for solving $\gr$-\relevance\ problem, especially for deciding the relevance of attacks related to a given set of arguments $S$. In order to decide the $\gr$-\relevance\ of an attack $r=(a,b)$ where $a,b\notin S$, we can use the same method for $\co$ in Theorem \ref{co_rel_theo}, i.e., $r\in RE^+(I,S,\gr$-$\true)$ iff $r\in RE^+(I,S,\co$-$\true)$. However, for attacks whose attackers or targets belong to $S$, the method of Theorem \ref{co_rel_theo} cannot be used to decide $\gr$-\relevance. This is because $\gr$ semantics requires {\it strong admissibility} \cite{baroni2007principle,caminada2020strong}, characterizing that the accepted set $S$ of arguments is able to be constructed from an initial set of arguments containing those in $S$ that are not attacked by any other arguments, and then repeatedly merging the arguments defended by the current set until there are no arguments that can be merged. Such a strong admissibility requirement leads deciding relevance for attacks related to $S$ under $\gr$ to be more difficult than under $\co$. For instance, consider the IAF $I=\langle\{a,b\},\emptyset,\{(a,b)\},\{(b,a)\}\rangle$. We can see that removal of $(b,a)$ is not $\co$-$\true$-relevant but $\gr$-$\true$-relevant for the set $\{a\}$. We leave the precise complexity of $\gr$-\relevance\ for future work.

    Stability decision problem in IAFs is first proposed in \cite{mailly2020stability} focusing on stability of acceptable status of an argument and related complexity results are given in \cite{baumeister2021acceptance}. Later a finer-grained acceptable status named justification status is introduced and the related stability problem is studied in \cite{odekerken2023justification}. Besides, {\it functionality} studied in \cite{alfano2022incomplete} presents a more strictly acceptable status for a single argument, which specifies whether an argument can be firmly decided in any completion. The stability of verification can be viewed as possible and necessary verification problems studied in \cite{baumeister2018verification,fazzinga2020revisiting}. In addition to stability of an single argument and a set of arguments, in our previous work \cite{xiong2024stability} we consider the stability of the whole extensions, i.e., an IAF is said to be stable for its extensions if any completion outputs the same extensions under certain semantics. Also focusing on the whole extensions of a semantics, \cite{skiba2020complexity} studies whether there exists a nonempty extension in some (or any) completion of an IAF. One can see that the notion of relevance and strong relevance can also be extended to these types of stability so as to solve their related problems about reaching stability.
    
    Detecting relevance and strong relevance can be applied to {\it enforcement} problems \cite{baumann2010expanding}, especially {\it general enforcement} proposed in \cite{coste2015extension}, which asks how to add arguments or modify attacks in the current AF to enable a set of arguments to become an extension. A dynamic AF under constrained update can be modeled as an IAF, where uncertain arguments represent the new arguments allowed to be added and uncertain attacks represent the original (un)attacks which are questionable. One can see that if addition or removal of an element $e$ is not -$\true$-relevant for the desired set of arguments $S$ w.r.t. the IAF, then such a modification is not included in any {\it minimal change} solution \cite{coste2015extension} for the enforcement of $S$. On the other hand, if addition or removal of $e$ is strongly -$\true$-relevant, then such a modification is necessary to be conducted in any minimal change. Also considering modification of one element $e$, the preservation problems \cite{rienstra2015persistence,rienstra2020principle} ask whether the verification status of an extension $S$ will remain after modification of $e$. Although they appear similar to the concept of `irrelevant' in this paper, they are essentially different because $e$ being irrelevant means that $S$ is preserved under modification of $e$ in all of the possible situations that would occur in the future (i.e., all of the completions of the given IAF). In \cite{oikarinen2011characterizing,baumann2012normal}, an operation named {\it $\sigma$-kernel} on AFs is proposed, which essentially selects and removes attacks useless for semantics $\sigma$ of an AF, i.e., the attacks whose modification has no influence on the output of $\sigma$-extensions of the AF. Actually, the operations for different semantics proposed in \cite{oikarinen2011characterizing,baumann2012normal} can be extended to IAFs and the attacks selected are irrelevant for verification of any set of arguments.  
\begin{table}[t]
    \centering
    \caption{Complexity of various types of relevance in IAFs.} 
    \begin{tabular}{|P{0.5cm}|P{2cm}|P{2cm}|P{2.5cm}|P{2.5cm}|}
        \hline
        \multirow{2}{*}{$\sigma$} & \multicolumn{2}{c|}{relevance for justification \cite{odekerken2023justification}} & relevance for verification & strong relevance for verification\\ \cline{2-3}
        \multirow{2}{*}{} & credulous-in & skeptical-in  & & \\ 
        \hline
        $\ad$ & $\Sigma_2^p$-$c$ & trivial & $P$ & $P$ \\ \hline 
        $\sta$ & $\Sigma_2^p$-$c$ & $\Sigma_2^p$-$c$  & $P$ & $P$ \\ \hline 
        $\co$ & $\Sigma_2^p$-$c$ & $NP$-$c$ & $P$ & $P$ \\ \hline 
        $\gr$ & $NP$-$c$ & $NP$-$c$ & in $NP$ & $P$ \\ \hline 
        $\pr$ & $\Sigma_2^p$-$c$ & $\Sigma_3^p$-$c$ & $\Sigma_2^p$-$c$ & $\Pi_2^p$-$c$, $coNP$-$c$\\ 
        \hline
    \end{tabular}
    \vspace{-6mm}
    \label{tab1} 
\end{table}
\section{Conclusions}
    This paper explores the relevance and strong relevance problems for stability of verification status of a set of arguments in IAFs. A complexity analysis of the decision problems of both kinds of relevance is conducted and all of the common semantics are given precise complexity except $\gr$ semantics. Such relevance problems are distinct from the relevance for justification of an argument in \cite{odekerken2023justification}, the original work where the relevance for stability is proposed and the study in this paper follows so as to tackle which uncertainties should be resolved to reach stability for verification. For future work, we will continue studying precise complexity for $\gr$ semantics, and extend the notion of (strong) relevance to other types of stability in IAFs including stability of the whole extensions.
    
\bibliographystyle{splncs04}
\bibliography{mybibliography}

%






\end{document}